\documentclass[twoside,11pt]{article}

\usepackage{blindtext}

\usepackage{mathtools} %
\usepackage{booktabs} %
\usepackage{tikz} %

\usepackage{amsthm}

\newtheorem{obs}{Observation}
\newtheorem{prop}{Proposition}
\newtheorem{lem}{Lemma}
\newtheorem{cor}{Corollary}
\newtheorem{remark}{Remark}
\usepackage{algpseudocode}
\usepackage{algorithm}
\usepackage{todonotes}
\usepackage{xcolor}
\usepackage{comment}

\usepackage[margin=1in]{geometry}

\usepackage{subcaption}

\usepackage{natbib} %

\usepackage{hyperref}
\usepackage{url}
\usepackage{graphicx}

\usepackage[capitalize]{cleveref}

\newcommand{\R}{\mathbb{R}}

\newcommand{\E}{\mathbb{E}}
\newcommand{\bigO}{\mathcal{O}}

\newcommand{\nnmu}{\hat{\mu}_\theta}
\newcommand{\nnLambda}{\hat{\Lambda}_\phi}
\newcommand{\ftmu}{\hat{\mu}}
\newcommand{\ftmud}{\vec{\mu}}

\newcommand{\ftp}{\hat{\Lambda}}
\newcommand{\ftpd}{\vec{\Lambda}}

\newcommand{\sep}{\!\;|\;\!}

\usepackage[american]{babel}

\usepackage{booktabs} %
\usepackage{placeins}
\usepackage{tabularx}
\usepackage{tikz} %

\usepackage[utf8]{inputenc} %
\usepackage[T1]{fontenc}    %
\usepackage{booktabs}       %
\usepackage{amsfonts}       %
\usepackage{nicefrac}       %
\usepackage{microtype}      %
\usepackage{xcolor}         %
\usepackage{amsmath}
\usepackage{enumitem}
\usepackage{array}
\usepackage{amssymb}

\usepackage{booktabs}

\crefname{obs}{Observation}{Observations}
\Crefname{obs}{Observation}{Observations}
\crefname{prop}{Prop.}{Props.}
\Crefname{prop}{Prop.}{Props.}
\crefname{lem}{Lemma}{Lemmas}
\Crefname{lem}{Lemma}{Lemmas}
\crefname{cor}{Cor.}{Cors.}
\Crefname{cor}{Cor.}{Cors.}
\crefname{rem}{Remark}{Remarks}
\Crefname{rem}{Remark}{Remarks}

\crefname{app}{Appendix}{Appendices}
\Crefname{app}{Appendix}{Appendices}

\usepackage{lastpage}

\title{On the Effect of Regularization on Nonparametric Mean-Variance~Regression}
\author{%
\begin{tabular}{cc}
Eliot Wong-Toi$^{1}$ & Alex Boyd$^{2}$ \\
\texttt{ewongtoi@uci.edu} & \texttt{alex.boyd@gehealthcare.com} \\[0.8em]
Vincent Fortuin$^{3,4}$ & Stephan Mandt$^{1}$ \\
\texttt{vincent.fortuin@tum.de} & \texttt{mandt@uci.edu}
\end{tabular}
}

\date{}  

\begin{document}
\maketitle

\begin{center}
{\small
$^{1}$University of California, Irvine\\
$^{2}$GE Healthcare\\
$^{3}$Technical University of Munich\\
$^{4}$Helmholtz AI
}
\end{center}
\maketitle

\begin{abstract}%
Uncertainty quantification is vital for decision-making and risk assessment in machine learning. Mean-variance regression models, which predict both a mean and residual noise for each data point, provide a simple approach to uncertainty quantification.
However, overparameterized mean-variance models struggle with signal-to-noise ambiguity, deciding whether prediction targets should be attributed to signal (mean) or noise (variance). 
At one extreme, models fit all training targets perfectly with zero residual noise, while at the other, they provide constant, uninformative predictions and explain the targets as noise. 
We observe a sharp phase transition between these extremes, driven by model regularization. 
Empirical studies with varying regularization levels illustrate this transition, revealing substantial variability across repeated runs.
To explain this behavior, we develop a statistical field theory framework, which captures the observed phase transition in alignment with experimental results.
This analysis reduces the regularization hyperparameter search space from two dimensions to one, significantly lowering computational costs. 
Experiments on UCI datasets and the large-scale ClimSim dataset demonstrate robust calibration performance, effectively quantifying predictive uncertainty.
\end{abstract}

\begin{quotation}
\noindent\textbf{Keywords:}
mean-variance model, heteroskedastic regression, phase transition, overfitting, regularization
\end{quotation}

\section{Introduction}\label{sec:introduction}

Deep learning models are now routinely trained in massively overparameterized regimes, where the number of parameters far exceeds the number of data points \citep{belkin_reconciling_2019,zhang_understanding_2021}. 
In many applications, such models perform reliably despite their overcapacity: they interpolate, generalize, and exhibit robust behavior under mild regularization \citep{krizhevsky_imagenet_2012,mathew_deep_2021}. 
However, there exists a growing class of models and objectives for which overparameterization interacts in delicate ways with the underlying likelihood, exposing structural instabilities that persist even when standard architectural heuristics such as batch normalization, dropout, or classical weight decay are applied. 
In these settings, the model is not failing to optimize its objective; rather, it is faithfully optimizing an objective that is intrinsically ill-conditioned \citep{nix_estimating_1994,seitzer_pitfalls_2022}.

Such behavior is not well captured by conventional statistical learning theory.
Classical analyses often rely on bounded-capacity assumptions, convexity, Lipschitz continuity, or uniform generalization bounds \citep{goos_rademacher_2001,mohri_foundations_2012,pmlr-v128-vapnik20a}, none of which adequately describe the collective behavior that emerges when many coupled degrees of freedom co-adapt during training. 
Instead, the phenomena resemble those studied in statistical physics, where abrupt qualitative changes in solution structure, sensitivity to small perturbations in hyperparameters, and competing ``forces'' between different components of the model are common \citep{landau_statistical_2013,altland_condensed_2010}. 
These parallels suggest that tools from the study of interacting systems, particularly phase transitions and variational principles, provide a more suitable lens \citep{ringel_applications_2025}.

An illustrative example of these ideas is overparameterized mean--variance regression, in which a model learns both a predictive mean and an input-dependent noise level \citep{nix_estimating_1994,nix_learning_1994}. 
In principle, this formulation should enable calibrated uncertainty quantification \citep{skafte_reliable_2019,fortuin_deep_2022}. 
In practice, it often displays striking instabilities \citep{seitzer_pitfalls_2022}. 
On one extreme, the mean network can nearly interpolate the data, pushing residuals and predicted variances toward zero. 
On the other, even modest regularization of the mean can cause the model to flatten its predictions and attribute nearly all structure to the variance \citep{stirn_faithful_2023, immer_effective_2023}. 
These behaviors arise across a range of architectures and optimization settings, and they appear as sharp transitions between qualitatively distinct regimes. 
As we show later, these extremes correspond closely to limits in which either the mean or variance channel is effectively unregularized. 
In such cases the objective becomes unbounded.

To move beyond empirical observations and toward an explanatory framework, we analyze mean--variance regression from a field-theoretic perspective. 
In this view, the learned mean and log-precision functions are treated as smooth fields governed by a free-energy-like functional \citep{altland_condensed_2010,landau_statistical_2013}. 
This continuum formulation abstracts away architectural details while preserving the essential interactions between data fit and regularization. 
Taking variational derivatives yields coupled Euler--Lagrange equations. 
These partial differential equations describe the stationary configurations of the model and explicitly reveal how the likelihood and regularization act as competing forces that redistribute prediction error across the input domain.

This field-theoretic analysis reveals several key properties that are not evident from neural experiments alone. 
Most notably, we show that the unregularized maximum likelihood regime admits no finite stationary solution, providing a mathematical explanation for the extreme overconfidence observed when the mean network interpolates the data \citep{zhang_understanding_2021}. 
We further show that one-sided regularization, in which only the mean or only the variance is penalized, renders the objective unbounded. 
This implies that two-sided regularization is not merely empirically helpful but structurally necessary. 
More broadly, the PDE structure predicts that solutions organize into qualitatively distinct regimes separated by sharp or smooth transitions, forming a \emph{phase diagram} over the space of regularization strengths.

We validate these predictions experimentally using synthetic datasets, standard UCI benchmarks \citep{kelly_uci_nodate}, and a large-scale climate simulation dataset \citep{yu_climsim_2023}. 
Across all settings, we observe a consistent qualitative picture: regimes of variance collapse, mean collapse, underfitting plateaus, and an intermediate region where both fields adapt stably to the data. 
Importantly, the solutions obtained from the neural models align closely with those predicted by the field theory, even though the latter is solved in a continuum limit.

Finally, our analysis suggests a natural reparameterization of the regularization strengths in which the effective balance between likelihood and smoothing is captured by a one-dimensional quantity.
This substantially simplifies hyperparameter tuning and provides a principled method for avoiding pathological regimes. 
The next section reviews prior work on mean-variance regression, statistical-physics perspectives on overparameterized models, and Bayesian formulations that motivate our field-theoretic approach.

A preliminary version of this work appeared as a conference paper \citep{wong-toi2024understanding}. 
The present manuscript substantially extends that version by developing a full field-theoretic formulation, establishing new well-posedness results, and introducing a Bayesian extension.

This work makes the following contributions:
\begin{enumerate}

    \item We develop a field-theoretic formulation of verparameterized mean--variance regression, treating learned predictors as continuous fields governed by a variational principle.

    \item We derive the coupled Euler--Lagrange equations that characterize stationary solutions and illustrate how likelihood and regularization act as competing forces on the mean and precision fields.
    
    \item We prove that the unregularized variational MLE functional has no finite minimizer, clarifying the mathematical source of the collapse behavior seen in neural mean--variance regression.

    \item We show that one-sided regularization renders the objective unbounded, demonstrating that both the mean and variance must be jointly regularized for the learning problem to be well-posed.

    \item We introduce a reparameterized regularization scheme that reduces the effective hyperparameter search to one dimension, yielding practical advantages for tuning and stability.

    \item We characterize the resulting phase transition structure of the solution space, identifying the main qualitative regimes and describing the boundaries that separate them.

    \item We introduce a Bayesian Field Theory (BFT) perspective by placing Gaussian process--like priors on the predictor fields, recovering the deterministic field theory as the posterior mode and providing a principled pathway for capturing epistemic uncertainty.

    \item We show that the Bayesian Field Theory connects naturally to the classical statistics literature on penalized likelihood, Gaussian process priors, and spline-based smoothing, placing overparameterized neural mean--variance regression within a unified framework that bridges modern deep learning and traditional nonparametric methods.
    
    \item We validate the theoretical predictions across neural implementations and numerical solutions of the field theory, observing strong agreement in both topology and transition structure.
\end{enumerate}

The code for all neural network and field-theoretic experiments 
is available at \url{https://github.com/ewongtoi/deep-heteroskedastic-regression}.
\section{Related Work}
\label{sec:related_work}
At a high level, our analysis connects several strands of prior work. Classical and early neural formulations of mean--variance regression identified degeneracies such as variance collapse and proposed architectural or weight-based regularization schemes to mitigate them \citep{nix_estimating_1994,bishop_mixture_1994,bishop_regression_1996,hjorth_regularisation_1999,li_degeneracy_2000}. 
Modern deep-learning approaches revisit these models in highly overparameterized settings, highlighting practical instabilities and exploring regularization grids, architectural constraints, and optimization strategies to stabilize maximum-likelihood training \citep{seitzer_pitfalls_2022,skafte_reliable_2019,stirn_faithful_2023,sluijterman_optimal_2024,immer_effective_2023}. 
Bayesian perspectives introduce smoothness priors or approximate posterior inference to improve uncertainty calibration \citep{yau_estimation_2003,yuan_doubly_2004,lakshminarayanan_simple_2017,stirn_variational_2020,pmlr-v119-wenzel20a,pmlr-v139-izmailov21a}. 
Our contribution differs in focus: we study the fully overparameterized regime at the level of function space, using tools from statistical physics and variational calculus to characterize when and why these instabilities arise and how regularization induces distinct solution phases.

\paragraph{Uncertainty and Mean-Variance Regression.}
Uncertainty is commonly divided into epistemic (model) and aleatoric (data) components \citep{hullermeier2021aleatoric}, with the latter decomposed into homoskedastic and heteroskedastic noise. Handling input-dependent noise has long been an active area in statistics \citep{huber_behavior_1967,eubank_detecting_1993,le_heteroscedastic_2005,uyanto_monte_2022} and machine learning \citep{abdar_review_2021}, but remains comparatively uncommon in modern deep learning \citep{kendall_what_2017,fortuin_deep_2022} due to the training instabilities studied in this work. 
Modeling heteroskedasticity can be interpreted as reweighting data points by their predictive uncertainty, a principle shown to improve robustness under label noise and imbalance \citep{mandt_variational_2016,wang_robust_2017,khosla_neural_2022}.

\paragraph{Connections to Statistical Physics.}
There is increasing interest in applying tools from statistical physics to machine learning, particularly for understanding generalization, loss landscapes, and phase transitions \citep{ringel_applications_2025}. 
Analogies to spin glasses and jamming phenomena reveal transitions between underfitting and overfitting and the emergence of multimodal energy surfaces \citep{franz_simplest_2016,geiger_jamming_2019,ros_complex_2019}. 
Related work on symmetry breaking and continuous symmetries in optimization \citep{bamler_improving_2018} highlights how physical principles can influence the structure of learning dynamics and the geometry of model families.
Sharp generalization transitions have been documented in regression and classification, including interpolation thresholds and double descent \citep{belkin_reconciling_2019,wu2023precise,george2023training,veiga_phase_2023}. 
Similar transition-like behavior appears in deep generative models \citep{wang_deep_2025}. 
Statistical mechanics approaches to inference \citep{Zdeborova_statistical_2016,antenucci_glassy_2019} further illuminate regimes where inference becomes computationally challenging. 
These perspectives motivate the field-theoretic description developed in our work.

\paragraph{Classical and Early Neural MVR.}
Neural mean--variance regression dates to the 1990s, when \citet{nix_estimating_1994} and \citet{bishop_mixture_1994} introduced Gaussian models that learn both mean and variance, noting degeneracies such as variance collapse. 
Subsequent work proposed unbiased variance estimators \citep{bishop_regression_1996}, Bayesian regularization via Gaussian weight priors \citep{hjorth_regularisation_1999}, and analyses of heteroskedastic GLMs exhibiting similar instabilities \citep{li_degeneracy_2000}. 
These studies show that ill-posedness is not unique to deep networks but arises whenever the likelihood can be increased by attributing residual structure to noise.

\paragraph{Modern Deep MVR and Regularization.}
Recent work has revisited these issues in overparameterized deep networks. 
\citet{seitzer_pitfalls_2022} analyze gradient blow-up as variances approach zero and propose reweighting schemes. 
\citet{skafte_reliable_2019,stirn_faithful_2023} decouple or constrain the variance pathway, while \citet{sluijterman_optimal_2024} map regularization grids and show that both channels must be regularized. 
New work on covariance and correlation learning \citep{pmlr-v235-shukla24a,shukla2025towards} further expands the space of heteroskedastic models.

\paragraph{Bayesian Perspectives.}
Bayesian neural networks \citep{mackay_practical_1992,neal2012bayesian,blundell2015weight,gal2016uncertainty} and approximate inference methods \citep{blundell2015weight,Welling2011BayesianLV,mackay_practical_1992,pmlr-v119-wenzel20a,pmlr-v139-izmailov21a} provide mechanisms for epistemic uncertainty. 
Bayesian formulations of MVR place priors on mean and variance functions \citep{yau_estimation_2003,yuan_doubly_2004} or infer precision parameters variationally \citep{stirn_variational_2020}. 
These works motivate the Bayesian Field Theory developed later, which places priors directly on the predictor functions and recovers the deterministic field theory as its posterior mode.

\section{Pitfalls of Overparameterized Mean--Variance Regression}
\label{sec:pitfall}

This section introduces heteroskedastic mean–variance regression, explains why the maximum-likelihood objective becomes ill-posed in overparameterized settings, and motivates the regularization schemes used throughout. We conclude by describing the qualitative \emph{phases} that arise across the regularization space.

\subsection{Heteroskedastic Mean--Variance Regression}

We consider independent data points $\mathcal D=\{(x_i,y_i)\}_{i=1}^N$ with covariates
$x_i \in \mathcal X \subset \mathbb R^d$ drawn from $p(x)$ and responses $y_i \in \mathbb R$.  
Throughout, we adopt the heteroskedastic Gaussian model
\begin{equation}
    p(y \mid x; \mu, \Lambda)
    = \mathcal N\!\big(y \sep \mu(x),\,\Lambda(x)^{-1}\big),
    \label{eq:hetero-gaussian}
\end{equation}
with mean function $\mu:\mathcal X\to\mathbb R$ and precision function $\Lambda:\mathcal X\to\mathbb R_{>0}$.

To develop the continuum (field-theoretic) formulation, we interpret the dataset as a single realization of the process described by~\eqref{eq:hetero-gaussian}. 
No smoothness is assumed of the true $(\mu,\Lambda)$; the analysis depends only on the observed realization $y(\cdot)$.  
For analytic convenience we assume $y\in H^1(\mathcal X)$. 
Later, after introducing the predictor functions $(\ftmu,\ftp)$, we impose $(\ftmu,\ftp)\in H^1(\mathcal X)^2$ and $\ftp(x)\ge c>0$ a.e.

\subsection{Overparameterized Neural Regression Models}

A common approach is to parameterize $\mu$ and $\Lambda$ using neural networks, whose universal approximation properties make them suitable for modeling flexible mean and precision functions \citep{hornik_approximation_1991}.  
Let $\nnmu:\mathcal X\to\mathbb R$ and $\nnLambda:\mathcal X\to\mathbb R_{>0}$ denote overparameterized feed-forward networks with parameters $\theta$ and $\phi$.  
We assume they do not share parameters, although shared encoders have been explored in prior work \citep[e.g.,][]{skafte_reliable_2019,seitzer_pitfalls_2022,stirn_faithful_2023}.  
For each $x_i$, these networks produce the values $\nnmu(x_i)$ and $\nnLambda(x_i)$ that define the likelihood in~\eqref{eq:hetero-gaussian}.

The population cross-entropy between $p(x,y)$ and the predictive distribution $\hat p(y\sep x)p(x)$ is
\begin{align}
\ell(\theta,\phi)
= H(p,\hat{p})
= -\E_{p(x,y)}\!\left[\log \hat{p}(y\sep x)\right].
\end{align}
Because the dataset $\mathcal D$ is a single Monte Carlo draw from $p(x,y)$, empirical 
training corresponds to replacing this expectation with its sample average. Replacing 
the expectation with this Monte Carlo estimate yields
\begin{align}
\ell_{MLE}(\theta,\phi)
= \frac{1}{2N}\sum_{i=1}^N
\Big[\nnLambda(x_i)\hat r_\theta(x_i)^2 - \log\nnLambda(x_i)\Big].
\label{eq:mle}
\end{align}

\subsection{Why Maximum Likelihood Fails}

Although~\eqref{eq:mle} is well defined for finite models, it becomes fundamentally ill-posed in the overparameterized setting \citep{nix_estimating_1994,bishop_mixture_1994,bishop_regression_1996,seitzer_pitfalls_2022}.  
The two terms in~\eqref{eq:mle} place contradictory pressures on $\nnLambda$: the residual term favors $\nnLambda\!\to\!0$, while the $-\log\nnLambda$ term favors $\nnLambda\!\to\!\infty$. This competition alone does not stabilize $\nnLambda$.  

At the same time, any sufficiently expressive mean network can interpolate at least one data point, driving $\hat r_\theta(x_i)\!\to\!0$. 
Once a residual vanishes, the first term in~\eqref{eq:mle} no longer restricts $\nnLambda(x_i)$, allowing it to diverge and driving the objective unbounded below. 
A related degeneracy arises in finite-dimensional heteroskedastic generalized linear models \citep{li_degeneracy_2000} whenever exact interpolation is possible.
Across both settings, the root cause is the same: the heteroskedastic Gaussian likelihood provides no mechanism to prevent collapse or explosion in $\Lambda$, and overparameterization exposes this structural weakness fully. 
This perspective explains why architectural heuristics such as weight decay or shared encoders may alleviate but do not eliminate the degeneracy. 
To better understand the interaction between likelihood and regularization, we next introduce explicit regularization schemes.

\subsection{Regularization in Mean--Variance Regression}

A natural remedy is to add $L_2$ penalties to both networks, as suggested by \citet{hjorth_regularisation_1999}, who also interpret these penalties as Gaussian weight priors:
\begin{align}
\ell_{\alpha,\beta}(\theta, \phi) 
:= \ell_{MLE}(\theta, \phi) + \alpha\|\theta\|_2^2 + \beta\|\phi\|_2^2,
\label{alphabeta}
\end{align}
where $\alpha,\beta \in \R_{\ge 0}$. 
Regularizing $\theta$ prevents mean overfitting, while regularizing $\phi$ prevents precision collapse and controls the complexity of the predicted uncertainty. 
As $\alpha\!\to\!\infty$, the mean becomes constant; as $\beta\!\to\!\infty$, the model becomes homoskedastic.
\footnote{Assuming an unpenalized bias term in the final layer or standardized data.}

Although intuitive, $(\alpha,\beta)$ spans an unbounded domain, so we adopt the bounded reparameterization
\begin{align}
\ell_{\rho,\gamma}(\theta, \phi) :=
\rho\, \ell(\theta, \phi) + \bar{\rho}\left[\gamma\,\|\theta\|_2^2 + \bar{\gamma}\,\|\phi\|_2^2\right],
\label{rhogamma}
\end{align}
with $\rho,\gamma\in(0,1)$ and $\bar{\rho}=1-\rho$, $\bar{\gamma}=1-\gamma$.  
This parameterization is one-to-one with $(\alpha,\beta)$, and both formulations yield proportional gradients.  
Here $\rho$ determines the trade-off between data fit and total smoothness, whereas $\gamma$ specifies whether smoothness is allocated primarily to the mean network or the precision network. 
The limits $\gamma=1$ and $\gamma=0$ correspond to unregularized precision and unregularized mean, respectively, and $\rho=1$ recovers pure MLE.

\begin{figure*}
\centering
\includegraphics[width=\textwidth]{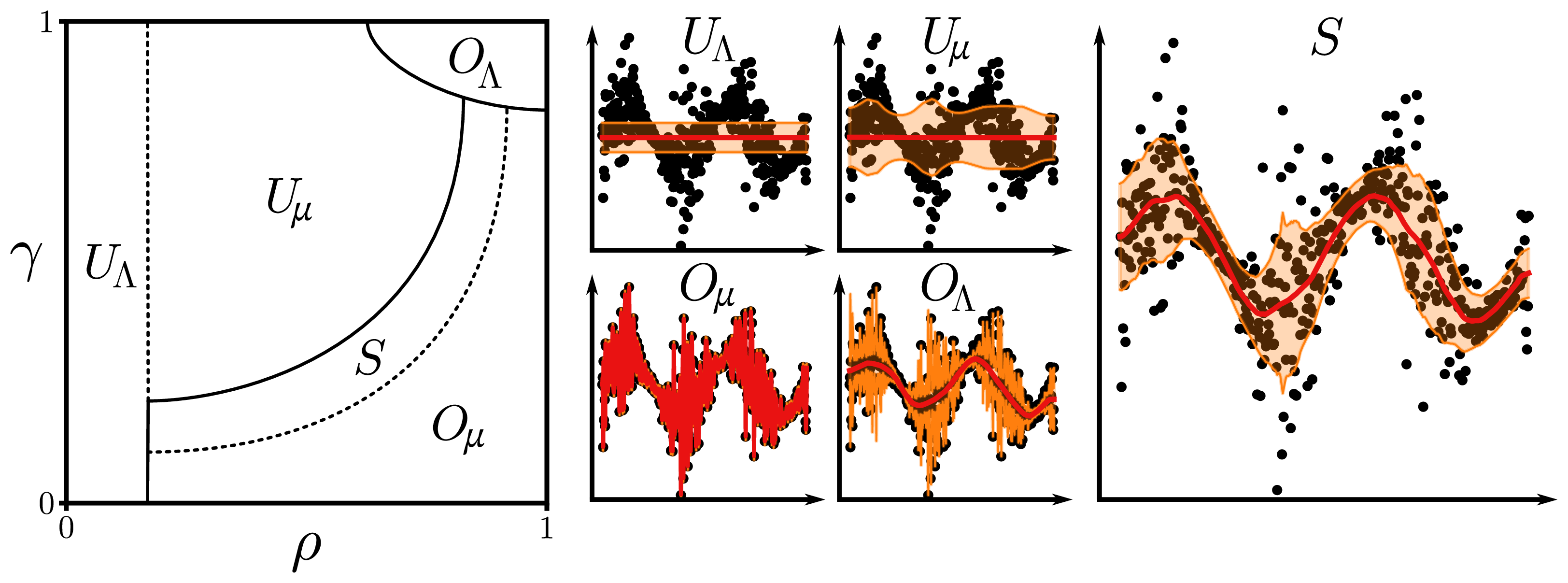}
\caption{
Phase-diagram sketch of mean–variance regression in the $(\rho,\gamma)$ plane (\emph{left}). 
Here $\rho$ controls the data-fit vs.\ smoothness trade-off, and $\gamma$ allocates smoothness between the mean and precision functions. 
Labeled regions indicate mean collapse ($O_\mu$), variance collapse ($O_\Lambda$), underfitting ($U_\mu$, $U_\Lambda$), and the stable regime $S$. 
Solid and dotted curves mark sharp and smooth transitions. 
Representative FT mean fits (red, with pointwise $\pm$~s.d.\ in orange) illustrate each regime 
(\emph{middle}, \emph{right}).
}
\label{fig:cartoonphases}
\end{figure*}

\subsection{Qualitative Phase Behavior}

Across the $(\rho,\gamma)$ space, the learned mean and precision functions exhibit a small number of recurring qualitative behaviors.  
These behaviors appear consistently across datasets, architectures, and both neural and field-theoretic solvers, suggesting an underlying structural organization.  
It is therefore useful to view the $(\rho,\gamma)$ plane as a \emph{phase diagram}, with representative solutions shown in \cref{fig:cartoonphases}.

\paragraph{Underfitting regimes.}
When $\rho$ is small, the solution is dominated by the regularization terms.  
In Region $U_{\Lambda}$, both functions remain close to constants, effectively ignoring the data.  
As $\rho$ increases and regularization is allocated primarily to the mean (Region $U_{\mu}$), the precision becomes more responsive, and the model explains variation through noise rather than signal.

\paragraph{Overfitting regimes.}
Collapse occurs when one network is effectively unregularized.  
In Region $O_{\mu}$, the mean nearly interpolates the data, leaving the precision largely irrelevant.  
In Region $O_{\Lambda}$, the precision adapts pointwise to fit residuals, producing spiky uncertainty estimates that mirror the data.

\paragraph{Stable regime.}
A comparatively narrow region $S$ exists between these extremes, where neither learned network collapses.  
Here the mean captures the dominant structure and the precision reflects heteroskedastic variation rather than residual interpolation.  
This region typically yields the most stable and well-calibrated solutions.

As we show in \cref{sec:field_theory}, the field-theoretic formulation exhibits analogous qualitative phases and offers partial insight into the behavior of their boundaries in several limiting regimes.

\section{Theoretical Considerations}
\label{sec:field_theory}

We now develop a theoretical description of how regularization strengths shape the behavior of heteroskedastic regression models. 
This framework captures the limiting behavior of neural networks in the fully overparameterized regime and allows us to analyze edge cases of regularization settings. 
In particular, it yields necessary conditions that any optimal pair of mean and standard deviation functions must satisfy, independent of architectural details. 
Numerical solutions of the resulting \emph{field theory}, introduced below, show strong qualitative agreement with practical neural network implementations.

Several of the structural transitions observed empirically in overparameterized networks resemble phase-transition phenomena analyzed in statistical physics. 
Prior work has documented analogous transitions in loss landscapes, jamming behavior, and generalization curves \citep{franz_simplest_2016, geiger_jamming_2019, ros_complex_2019, belkin_reconciling_2019}. 
Our field-theoretic formulation makes these parallels explicit by showing how competing likelihood and regularization terms give rise to sharp and smooth transitions in the learned predictors.

\subsection{Field Theory}
Having discussed the qualitative effects of regularization on a mean-variance model, we examine whether the extent to which these effects depend on a specific neural network architecture. 
Furthermore, we explore whether some of these effects can be characterized at the function level, independent of neural networks. 
To address these points, we develop \emph{field theories} inspired by statistical mechanics.

Field theories are statistical descriptions of random functions, rather than discrete or continuous random variables~\citep{altland_condensed_2010}. 
A \emph{field} is a function from spatial coordinates to scalar values (or vectors). 
Examples include electric charge density or a surface height. 
In field theory, we often seek the configuration (function) that minimizes an energy functional.
Low-energy configurations of fields can display recurring patterns (e.g., waves) or undergo phase transitions (e.g., magnetism) upon varying model parameters.  
Since we can think of a function as an infinite-dimensional vector, field theory requires the usage of \emph{functional analysis} over plain calculus. 
For example, we frequently ask for the field that minimizes a free energy functional that we obtain by calculating a functional derivative that we set to zero. 
The advantage of moving to a function-space description is that all details about neural architectures are abstracted away as long as the neural network is sufficiently over-parameterized. 

Firstly, we abstract the neural networks $\nnmu$ and $\nnLambda$ by smooth nonparametric functions $\ftmu$ and $\ftp$ in the continuum limit; the exact Sobolev regularity conditions are specified in Appendix~\ref{app:theoretical_details}. 
Though an ``$L_2$'' penalty does exist on a functional level (and could be applied to the FT), it penalizes the magnitude of the output of the function which is an inherently different aspect of the function. 
This would encourage predictions near zero rather than simple behaviors. 

A comparable alternative is to directly penalize the output ``complexity'' of the models, measured via the \emph{Dirichlet energies}, 
$\!\int\! p(x)\|\nabla \ftmu(x)\|_2^2\,dx$ and 
$\!\int\! p(x)\|\nabla \ftp(x)\|_2^2\,dx$, 
corresponding to the respective mean and precision functions which we can weight via $\rho, \gamma$ (or $\alpha, \beta$ as we did in the earlier neural network objective.
A similar approach is taken by \citet{yuan_doubly_2004} where they apply separate roughness penalties to mean and noise functions. 
These quantities can be computed without any assumption on the particular parameterization of the functions.
Note that these specific penalizations induce similar limiting behaviors for resulting solutions---$\rho \!\to\! 0$ implies functions that are quickly changing direction (overfitting) while $\rho \to\! \infty$ implies constant functions (underfitting). 
The discrete analogue to \emph{Dirichlet energy} is the \emph{geometric complexity (GC)}
\begin{align}
    GC(f, \mathcal{D}) = 
    \frac{1}{|\mathcal{D}|} \sum_{i=1}^{|\mathcal{D}|} \|\nabla_x f(x)\|^2_F
\end{align}
where $\|\nabla_x f(x)\|^2_F$ is the Frobenius norm of the network Jacobian. 
The geometric complexity of a neural network has been found to be lower in the presence of stronger $L_2$ penalties \citep{dherin_why_2022} and in the case where $\nnmu$ and $\nnLambda$ are linear models, this gradient penalty is equivalent to an $L_2$ penalty. 

Using the assumptions outlined above and the same reparameterization of 
$(\alpha, \beta)$ to $(\rho, \gamma)$ as in the neural network setting, 
the penalized cross-entropy can be interpreted as the action functional of a 
corresponding two-dimensional field theory (FT):
\begin{align}
\mathcal{S}_{\rho,\gamma}[\ftmu,\ftp]
&= 
\int_{\mathcal{X}} p(x) \left\{\rho\!
   \int_{\mathcal{Y}} p(y\sep x)[-\log \hat{p}(y\sep x)]\,dy
+ \,\bar{\rho}\!\left[
   \gamma\lVert\nabla\ftmu(x)\rVert_2^2
 + \bar{\gamma}\lVert\nabla\ftp(x)\rVert_2^2
 \right]\right\}dx,%
\label{eq:ft_def}
\end{align}
where $\hat{p}(y\sep x)=\mathcal{N}(y\sep\ftmu(x),\ftp(x)^{-1})$. 
This formulation assumes continuous densities $p(x)$ and $p(y\sep x)$ 
and continuous predictor functions $\ftmu(x)$ and $\ftp(x)$.

\subsubsection{Empirical (sampled) field theory}

Let $y(\cdot)=\{y(x)\}_{x\in\mathcal{X}}$ denote a realization of the stochastic process $y(x)\sim\mathcal N(\mu(x),\Lambda(x)^{-1})$, and assume that the realized data field satisfies $y\in H^1(\mathcal X)$.  
In keeping with standard statistical and machine learning practice, all inference is performed \emph{conditional on a single observed dataset}.
Accordingly, we work with one realization $y(\cdot)$ rather than taking an expectation over multiple draws.
This also avoids the computational burden of integrating over repeated noise realizations, yielding the \emph{empirical} field-theoretic functional:
\begin{align}
    \mathbf{S}_{\rho,\gamma}[\ftmu,\ftp]
    &=
    \int_{\mathcal{X}} p(x)\,
    \Big\{
    \rho\!\left[\tfrac{1}{2}\ftp(x)\hat{r}^2(x)
    -\tfrac{1}{2}\log\ftp(x)\right]
    +\bar{\rho}\!\left[
    \gamma\lVert\nabla\ftmu(x)\rVert_2^2
    +\bar{\gamma}\lVert\nabla\ftp(x)\rVert_2^2
    \right]
    \Big\}\,dx,
\end{align}
where $\hat{r}(x)=\ftmu(x)-y(x)$ is the pointwise residual field.
By construction,
\begin{equation}
    \mathbf{S}_{\rho,\gamma}[\ftmu,\ftp]
    \;\approx\;
    \mathcal{S}_{\rho,\gamma}[\ftmu,\ftp] + K,
    \qquad
    \text{where $K$ is independent of $(\ftmu,\ftp)$}.
\end{equation}
The approximation arises solely from replacing the population expectation $\mathbb{E}_{p(y\mid x)}[\cdot]$ by its empirical counterpart under the observed realization $y(\cdot)$. 
This mirrors the finite-sample replacement used in the maximum-likelihood formulation: just as the empirical MLE objective approximates the population cross-entropy, the empirical functional $\mathcal{S}_{\rho,\gamma}$ approximates the population-level field-theoretic functional $\mathbf{S}_{\rho,\gamma}$.

\textbf{Remark.}
The conditional Gaussian model
$p(y\mid x)=\mathcal{N}(y\mid\mu(x),\Lambda(x)^{-1})$
is an \emph{exact} specification of the data–generating process.
Thus the field-theoretic functional itself remains exact under conditional normality; the symbol ``$\approx$'' reflects only the finite-sample (Monte Carlo) replacement of the population expectation by the single observed dataset, which is both conceptually standard and computationally advantageous.

\subsubsection{Function-space setting}
\label{sec:function_space_setting}

We briefly record the regularity assumptions under which the functional $\mathbf{S}_{\rho,\gamma}[\ftmu,\ftp]$ is well defined and its variational derivatives can be computed. Throughout, $\mathcal X\subset\mathbb R^d$ denotes a bounded Lipschitz domain and $p\in C^1(\overline{\mathcal X})$ is strictly positive on $\overline{\mathcal X}$, so that $p$ is bounded above and below by positive constants.

We restrict the admissible fields to the usual Sobolev space $H^1(\mathcal X)$, requiring square-integrability of both the functions and their weak gradients:
\[
    (\ftmu,\ftp) \in H^1(\mathcal X) \times H^1(\mathcal X),
    \qquad\text{with}\qquad \ftp(x) \ge c > 0 \ \ \text{a.e.\ in }\mathcal X.
\]
The lower bound on the precision ensures non-degeneracy of the variance $\ftp^{-1}$ and, in particular, guarantees that the data-fidelity term $-\log \ftp$ remains finite. 
Under these assumptions, the weighted Dirichlet energies
\[
\int_{\mathcal X} p(x)\,\|\nabla f(x)\|_2^2\,dx
\]
are finite for all admissible $f\in H^1(\mathcal X)$, since the positivity and essential boundedness of $p$ imply the equivalence of the weighted and unweighted $L^2$ norms on $\mathcal X$.

From the perspective of ML and statistics, these conditions prevent degenerate behaviors such as vanishing or exploding variances and exclude functions with unbounded roughness. 
Consequently, the regularization term behaves as a well-posed smoothness penalty, while the likelihood contribution remains stable and well defined.

Taken together, these assumptions ensure that the functional $\mathbf{S}_{\rho,\gamma}[\ftmu,\ftp]$ is finite for all admissible pairs $(\ftmu,\ftp)$ and that its associated Euler--Lagrange equations are meaningful in the usual weak sense. 
Equivalently, the variational problem admits well-defined first-order optimality conditions that correspond to bona fide partial differential equations rather than ill-posed distributional identities.

\subsection{Field-Theoretic Stationarity and Insights}
\label{sec:ft_insights}
Taking variational derivatives of the field-theory objective with respect to the mean and precision functions and imposing homogeneous Neumann boundary conditions yields the following \emph{stationary conditions}, which hold almost everywhere with respect to the data density~$p(x)$:
\begin{subequations}\label{eq:pdes_revised}
\begin{align}
\rho\,p(x)\,\ftp^*(x)\,\hat r^*(x)
&= 2\bar\rho\,\gamma\,\mathcal{L}_p \ftmu^*(x),\\[0.4em]
\tfrac{\rho}{2}\,p(x)\!\left[\hat r^{*}(x)^2 - \tfrac{1}{\ftp^*(x)}\right]
&= 2\bar\rho\,\bar\gamma\,\mathcal{L}_p \ftp^*(x),
\end{align}
\end{subequations}
where $\hat r^{*}(x)=\ftmu^{*}(x)-y(x)$ and $\mathcal{L}_p f := -\,\nabla\!\cdot\!\big(p(x)\nabla f\big)$ denotes the (unnormalized) weighted Laplacian introduced in \Cref{lem:weighted_greens}.  
When $p(x)$ is constant, $\mathcal{L}_p$ reduces to the standard Laplacian $\Delta$ \citep{engel_density_2011}, recovering the uniform-density setting studied in earlier work~\citep{wong-toi2024understanding}.  
These equations coincide with the Euler--Lagrange conditions formalized in \Cref{prop:generalFT} of Appendix~\ref{app:theoretical_details}.  
Below we summarize the theoretical implications obtained by taking limits over $(\rho,\gamma)$.

\begin{prop}[Combined analytical structure of MVR]
\label{prop:summary_main}
Assume that $\mathcal X \subset \mathbb R^d$ is a bounded, connected Lipschitz domain, that $p \in C^1(\overline{\mathcal X})$ is strictly positive on $\overline{\mathcal X}$, and that the data field satisfies $y \in H^1(\mathcal X)$. 
Let $(\ftmu,\ftp) \in H^1(\mathcal X) \times H^1(\mathcal X)$ with $\ftp(x) \ge c > 0$ for a.e.\ $x \in \mathcal X$. 

Then any stationary point of the functional $\mathbf S_{\rho,\gamma}[\ftmu,\ftp]$ satisfies the Euler--Lagrange system~\eqref{eq:pdes_revised} under homogeneous Neumann (zero-flux) boundary conditions $p \nabla \ftmu \cdot \mathbf n = 0$ and $p \nabla \ftp \cdot \mathbf n = 0$ on $\partial\mathcal X$.
Furthermore:
\begin{enumerate}[label=(\roman*)]
    \item For $\rho = 1$ (no regularization), the EL equations admit no stationary solution.
    
    \item For $\rho = 0$ (no data term), the minimizer is non-unique: any constant 
    pair $(\ftmu,\ftp)$ satisfies the Neumann equations on a connected domain.
    
    \item For $\gamma = 0$ or $\gamma = 1$ and $\rho>0$, the functional 
    $\mathbf S_{\rho,\gamma}$ is unbounded below; hence both regularization channels 
    must be strictly positive for well-posedness.
    
    \item If in addition the precision is uniformly bounded 
    $0 < c \le \ftp(x) \le \lambda_{\max} < \infty$ a.e.\ in $\mathcal X$, 
    then for all $\rho,\gamma \in (0,1)$ the variational problem admits at least one 
    solution $(\ftmu^\ast,\ftp^\ast) \in H^1(\mathcal X)^2$; see 
    \Cref{rem:bounded_prec} for discussion of these additional assumptions.
\end{enumerate}
Hence, well-posed formulations require $\rho\!\in\!(0,1)$ and $\gamma\!\in\!(0,1)$, or equivalently two-sided positive penalties $\alpha,\beta>0$.
\end{prop}

The proofs, based on weighted Green’s identities and weak Euler–Lagrange arguments, are given in Appendix~\ref{app:theoretical_details} (see \cref{prop:generalFT,prop:general_extremes,cor:natural_neumann,prop:interior_exist}).

\subsubsection{Interpretation of Prop. 1}
Under the assumptions of \Cref{prop:summary_main}---namely, smooth positive $p(x)$ on a compact domain $\mathcal X\subset\mathbb R^d$ with homogeneous zero-flux boundaries---any stationary point $(\ftmu^*,\ftp^*)$ of the field-theoretic objective satisfies the coupled PDEs in \eqref{eq:pdes_revised}.
Moreover, no solution with finite objective value exists in the limiting case $\rho\!\to\!1$, consistent with the unregularized regime analyzed in \Cref{prop:general_extremes}.

Each equation in \eqref{eq:pdes_revised} expresses a balance between a data-fitting term (on the left) and a regularization-induced flux (on the right).
The divergence operator plays a diffusion-like role, redistributing prediction errors across the input domain rather than letting them concentrate at isolated points.
The presence of $p(x)$ in the denominators modulates this diffusion according to the local data density:
regions with higher $p(x)$ experience a weaker effective regularization, while low-density regions experience a stronger one.
Equivalently, the local smoothing strength is proportional to
\[
\text{effective regularization strength} 
\;\propto\;
\frac{\bar\rho}{\rho}\,\gamma\,p(x)^{-1}.
\]
Areas containing many data points therefore permit more functional complexity, whereas sparsely sampled regions are forced toward smoother, simpler predictions.
This adaptive weighting makes the field theory sensitive to the empirical geometry of the data---a property absent from standard, unweighted $L_2$ regularization.

The weighted Laplacians $\mathcal{L}_p\ftmu$ and $\mathcal{L}_p\ftp$ quantify the local curvature of the mean and precision functions under the data density $p(x)$.
Accordingly, the hyperparameters $\rho$ and $\gamma$ directly determine how much curvature the model can sustain.
Recall that $\rho \in (0,1)$ controls the overall balance between likelihood fitting and regularization (larger $\rho$ emphasizes data fidelity, smaller $\rho$ enforces stronger smoothing), while $\gamma \in (0,1)$ partitions the total regularization between the mean and precision fields ($\gamma$ weighting the mean term and $\bar\gamma := 1-\gamma$ weighting the precision term).
Large $\rho$ (weak regularization) sharpens curvature and risks overconfident solutions, whereas small $\rho$ (strong regularization) flattens both functions, leading to underfitting.
Similarly, varying $\gamma$ trades off smoothness between the mean and variance functions.
These coupled PDEs therefore formalize the intuitive ``phase diagram'' of \Cref{fig:cartoonphases}, in which distinct regimes of $(\rho,\gamma)$ correspond to qualitatively different equilibrium configurations of the fields.

The derivation of \eqref{eq:pdes_revised} and the associated boundary conditions follows from the weighted Green’s identity in \Cref{lem:weighted_greens} and the Euler--Lagrange system proved in \Cref{prop:generalFT}.
Extreme and degenerate cases are analyzed in \Cref{prop:general_extremes}, which demonstrate the absence or unboundedness of solutions when either $\rho$ or $\gamma$ lies at the boundary of $[0,1]$.
Finally, \Cref{cor:two_sided_needed_general} shows that both mean and variance regularization terms must be strictly positive ($\gamma\in(0,1)$) to ensure well-posedness.
Together, these results provide the analytic basis for the phase-transition structure %
schematized in \Cref{fig:cartoonphases}.

These limiting cases align with the intuition conveyed earlier and apply equally in the neural network setting.  
Assuming valid stationary solutions exist for $\rho,\gamma\in(0,1)$, one expects the system to exhibit either sharp transitions or smooth cross-overs between the behaviors described in the extremes as the regularization strengths vary.  
Empirically, Section~\ref{sec:experiments} shows that the resulting phase diagrams resemble \Cref{fig:cartoonphases}.  
A complete analytical justification for the boundary types, their shapes, and precise placement within the $(\rho,\gamma)$ plane is left for future work.

While the field theory is deterministic, its nonconvex objective can admit multiple distinct minimizers. 
This multiplicity reflects structural ambiguity in the learned predictor functions and provides a natural interpretation of \emph{epistemic uncertainty}---not as stochasticity in the model itself, but as sensitivity to initialization and optimization. 
In particular, the existence results in \Cref{prop:generalFT,prop:interior_exist} do not guarantee uniqueness of minimizers, so different runs may converge to qualitatively distinct solutions. 
This perspective motivates the ensemble-based approximation introduced in \Cref{sec:bft_sampling}, in which variability across local minimizers approximates the posterior spread of the Bayesian field theory (\cref{sec:bft_main}).

\subsection{Numerically Solving the FT}
\label{sec:deterministic_discretization}

We describe the numerical procedure used to approximate minimizers of the deterministic field theory (FT) objective introduced in~\cref{sec:ft_insights}.  
Because the Euler--Lagrange equations in~\eqref{eq:pdes_revised} rarely admit closed-form solutions, we work with a discrete approximation of the continuous functional~$\mathbf S_{\rho,\gamma}[\ftmu,\ftp]$.  
In practice we restrict attention to one-dimensional domains~$\mathcal X\subset \mathbb R$.

\subsubsection{Uniform lattice discretization}

Let  
$\mathsf L^{(D)}=\{x^{(D)}_i\}_{i=1}^{D}$
denote a uniform grid on $\mathcal X$ with spacing $h\approx |\mathcal X|/D$.  
For boundary points, we pad one extra point on each side so that centered finite differences can be used at all interior indices; one-sided differences at the padded nodes enforce homogeneous Neumann boundary conditions.
We define the discrete vectors
\[
    \ftmud^{(D)} = (\ftmu(x^{(D)}_i))_{i=1}^D,\qquad
    \ftpd^{(D)} = (\ftp(x^{(D)}_i))_{i=1}^D,\qquad
    \vec y^{(D)} = (y(x^{(D)}_i))_{i=1}^D,
\]
and let $\nabla_h$ be the standard centered finite-difference gradient.
The discrete FT objective is
\begin{align}
    \mathsf S^{(D)}_{\rho,\gamma}(\ftmud,\ftpd)
    =
    \frac{1}{D}\sum_{i=1}^{D}
    \Bigg\{
    \rho\!\left[
    \frac12\,\ftpd_i\,(y_i-\ftmud_i)^2
    - \frac12\log \ftpd_i
    \right]
    +
    \bar\rho\!\left[
    \gamma\,\|\nabla_h\ftmud\|_i^2+
    \bar\gamma\,\|\nabla_h\ftpd\|_i^2
    \right]
    \Bigg\}.
\label{eq:disc_ft}
\end{align}
We minimize~\eqref{eq:disc_ft} via gradient descent to obtain lattice approximations of $\ftmu$ and $\ftp$.

\subsubsection{Consistency with the continuous FT}

Because $\mathsf L^{(D)}$ is a uniform grid rather than a random sample from $p(x)$, convergence of the discrete objective to $\mathbf S_{\rho,\gamma}$ follows from deterministic quadrature rather than Monte--Carlo averaging.  
If $\ftmu,\ftp\in H^1(\mathcal X)$ and $p$ is continuous and bounded above and below, then Riemann sum approximations yield
\[
    \lim_{D\to\infty}
    \mathsf S^{(D)}_{\rho,\gamma}\bigl(\Pi_D\ftmu,\Pi_D\ftp\bigr)
    =
    \mathbf S_{\rho,\gamma}\bigl[\ftmu,\ftp\bigr],
\]
where $\Pi_D$ denotes projection of the continuous fields onto the grid.  
Convergence of the discrete gradient terms follows from standard finite-difference consistency on uniform meshes~\citep{fornberg_generation_1988, brenner_mathematical_2008}.
Violations of the regularity assumptions (e.g.\ if $\nabla\ftp$ is unbounded on a set of nonzero measure) may lead to instability or divergence of the discrete minimizer.

\subsubsection{Relation to the general finite-element setting}

The uniform grid considered here represents the one-dimensional analogue of the general mesh $\mathcal G_h$ in Appendix~\cref{app:bft_discrete_appendix}.  
In higher dimensions, $\mathcal G_h$ is a shape-regular collection of elements supporting finite-element approximations of the weighted Laplacian $\mathcal L_p f=-\nabla\!\cdot(p\nabla f)$.  
Uniform grids in 1D can be viewed as a special case of such meshes with $p$ incorporated solely through quadrature weights.  
Thus, standard FEM convergence theory applies directly~\citep{brenner_mathematical_2008}.

\section{Bayesian Reformulation of the Field Theory}
\label{sec:bft_main}

The deterministic field theory (FT) introduced in \cref{sec:field_theory} can be extended to a fully probabilistic formulation by placing priors directly on the mean and log-precision fields $(\ftmu, \ftp)$. 
This yields a \emph{Bayesian Field Theory} (BFT) in which the FT energy functional appears as a log-posterior, the deterministic solution arises as a maximum a posteriori (MAP) estimate, and posterior samples correspond to stochastic field realizations that quantify uncertainty.

This function-space viewpoint is closely related to classical Bayesian treatments of heteroskedastic regression. 
Prior statistical work places smoothness priors on spline-based mean and variance functions \citep{yau_estimation_2003, yuan_doubly_2004}, while \citet[Section~3.7.1]{lemm_bayesian_2000} analyzes homoskedastic Gaussian regression as a field theory with Gaussian priors defined by differential operators. 
The BFT developed here provides a continuous analogue adapted to the heteroskedastic setting and offers a principled route to uncertainty quantification that complements weight-space Bayesian neural network approaches.

We write the negative log-posterior (up to a constant) as
\begin{equation}
\Phi(\ftmu; \ftp)
=
\underbrace{
  -\log \pi(\ftmu,\ftp)\vphantom{\int_{\mathcal X}}
}_{\text{Prior}}
\;
-\;
\underbrace{
  \rho \!\int_{\mathcal X}\! p(x)\log\hat{p}(y \sep x)\,dx
}_{\text{Likelihood term}},
\qquad
\hat{p}(y\sep x)
= \mathcal{N}\!\big(y\sep\ftmu(x),\ftp(x)^{-1}\big)
\label{eq:bft_general}
\end{equation}
where $\pi(\ftmu,\ftp)$ encodes smoothness priors over the predictor functions.
Here $\rho$ acts as a relative weighting (or inverse temperature) on the likelihood term, and the overall scaling between likelihood and prior is arbitrary up to a constant factor.
All integrals are taken over the input domain $\mathcal{X}$ with respect to the normalized density $p(x)$ and can thus be interpreted as expectations under the input distribution.

Here $\ftmu(x)$ denotes the predictive mean function and $\ftp(x)>0$ the predictive \emph{precision} (i.e., inverse-variance). 
We also define $\hat\eta(x)=\log\hat\Lambda(x)$ as the \emph{log-precision} field, which is used in the subsequent parameterizations and enforces positivity after exponentiation. 
This notation ensures that the Gaussian likelihood term $\tfrac{1}{2}\hat\Lambda(y-\hat\mu)^2-\tfrac{1}{2}\log\hat\Lambda$ matches the standard negative log-likelihood of $\mathcal N(y\!\mid\!\hat\mu,\hat\Lambda^{-1})$ up to an additive constant.

\subsection{MAP-Equivalent Prior}
\label{sec:bft_additive}
As before, we write $\bar\rho = 1 - \rho$ and $\bar\gamma = 1 - \gamma$.
To recover the deterministic FT exactly, we parameterize
$\ftp = e^{\hat{\eta}}$,
where $\hat{\eta}$ is the \emph{model} log-precision function, and choose priors that yield the same gradient penalties as the FT functional.

The prior distributions are chosen to impose smoothness on the predictor functions, mirroring the geometry of the deterministic FT regularizers. 
We place a Gaussian field prior on the mean function to enforce smoothness and a log–convex prior on the log–precision to ensure positivity and stability of the noise function:
\begin{subequations}\label{eq:bft_priors_main}
\begin{align}
-\log \pi(\ftmu)
&= \tfrac{\gamma}{2}\!\int_{\mathcal X}\! p(x)\,\|\nabla\ftmu(x)\|_2^2\,dx,
\label{eq:bft_mu_prior_main}\\[0.3em]
-\log \pi(\hat\eta)
&= \tfrac{\bar\gamma}{2}\!\int_{\mathcal X}\! p(x)\,e^{2\hat\eta(x)}\|\nabla\hat\eta(x)\|_2^2\,dx,
\label{eq:bft_eta_prior_main}
\end{align}
\end{subequations}
where $\hat\eta := \log \ftp$ ensures $\ftp > 0$.

Since $\nabla\ftp = \nabla(e^{\hat\eta}) = e^{\hat\eta}\nabla\hat\eta$, we have $\|\nabla\ftp\|^2 = e^{2\hat\eta}\|\nabla\hat\eta\|^2$, so the penalty in~\eqref{eq:bft_eta_prior_main} is exactly the FT Dirichlet energy $\tfrac{\bar\gamma}{2}\!\int p\,\|\nabla\ftp\|_2^2\,dx$ expressed in the log–precision parameterization. (We adopt the same homogeneous weighted Neumann boundaries as in the FT derivation and omit them here for brevity; see \Cref{cor:natural_neumann} for details.)

Combining these priors with the likelihood yields the full posterior energy functional,
whose stationary conditions coincide with those of the deterministic FT:
\begin{equation}
\Phi_{\text{MAP}}(\ftmu,\hat\eta)
=\int_{\mathcal X} p\,\Big\{\rho\big[\tfrac{1}{2}e^{\hat\eta}(y-\ftmu)^2-\tfrac{1}{2}\hat\eta\big]
+\tfrac{\bar\rho}{2}\big[\gamma\|\nabla\hat\mu\|^2+\bar\gamma\,e^{2\hat\eta}\|\nabla\hat\eta\|^2\big]\Big\}.
\end{equation}
For brevity, we omit explicit dependence on $x$ where unambiguous.
This functional coincides with $\mathbf{S}_{\rho,\gamma}$ up to constants, so minimizing $\Phi_{\text{MAP}}$ is precisely the FT optimization problem.  
The posterior mode thus satisfies the same stationary equations as \cref{prop:generalFT}, confirming the equivalence $\text{MAP}\equiv\text{FT}$. 
(Up to the conventional $\tfrac{1}{2}$ scaling in Gaussian log--densities; see Appendix~\ref{app:bft_appendix} for discussion.)

Our focus in this section is primarily analytical rather than computational: we study the structure of the Bayesian field theory (BFT) and its direct relationship to the deterministic FT, rather than performing full posterior sampling.  
Nevertheless, the Bayesian formulation clarifies how uncertainty arises in function space and provides a principled foundation for later ensemble-based approximations (see \cref{sec:bft_sampling}).

\subsubsection{Connection to Gaussian processes}

Before turning to computational approximations, it is helpful to relate the Bayesian field theory (BFT) to Gaussian processes (GPs), since both describe random functions and both encode smoothness through quadratic penalties.

A GP specifies its behavior directly through a covariance kernel $k(x,x')$.
In contrast, a Gaussian field encodes smoothness implicitly through a linear differential operator~$\mathcal L$ that penalizes roughness, much like a
classical spline penalty in reproducing-kernel Hilbert spaces~\citep{kimeldorf_wahba_bayes_1970,Wahba90}.  
The covariance structure of the field is then determined by the inverse of this operator---specifically, by the Green’s function solving $\mathcal L G = \delta$. 
This operator–kernel relationship, which plays a central role in modern SPDE-based Gaussian field models \citep{lindgren_explicit_2011}, provides the bridge between the field-theoretic and GP viewpoints.  
A formal derivation is given in Appendix~\ref{app:bft_gp_equivalence}.

In our setting, the relevant operator is the weighted Laplacian $\mathcal L_p f = -\nabla\!\cdot(p\nabla f)$.  
With homogeneous Neumann boundary conditions, this operator leaves constant functions unchanged, and the corresponding prior is therefore ``intrinsic'' (improper) up to an additive constant.  
This is a standard phenomenon in intrinsic Gaussian Markov random fields \citep{rue2005gaussian,rue_approximate_2009} and does not cause practical difficulty: the likelihood (or centering) fixes the overall offset of $\hat\mu$, yielding a proper posterior.  
Equivalently, the associated RKHS is the quotient space $H^1(\mathcal X)/\{\text{constants}\}$.

For example, a prior of the form
\[
    p(\hat\mu)\propto 
    \exp\!\left[-\tfrac{\gamma}{2}\!\int_{\mathcal X} p(x)\,\|\nabla\hat\mu(x)\|^2
    \, dx\right]
\]
corresponds to a Gaussian field whose precision is $\gamma\mathcal L_p$, and hence to a GP whose covariance is $(\gamma\mathcal L_p)^{-1}$. 
In this view, the Dirichlet energy simply plays the role of a smoothness penalty: in the GP formulation it arises from the covariance kernel, and in the field-theoretic formulation it arises from the corresponding precision operator. 
These two descriptions are mathematically equivalent~\citep{kimeldorf_wahba_bayes_1970,Wahba90,lindgren_explicit_2011}.

\subsection{Sampling and Approximation}
\label{sec:bft_sampling}
In principle, posterior samples of the continuous functions can be generated using
stochastic-gradient Langevin dynamics (SGLD), which treats the MAP functional
as an energy landscape and simulates noisy gradient descent \citep{Welling2011BayesianLV}:
\[
z_{t+1}=z_t-\epsilon_t\,\widehat{\nabla}\Phi_{\text{MAP}}(z_t)
+\sqrt{2\epsilon_t}\,\xi_t,
\quad
\xi_t\!\sim\!\mathcal N(0,I),\ 
\sum_t\epsilon_t=\infty,\ \sum_t\epsilon_t^2<\infty.
\]
Here $z_t$ denotes the concatenated discretization of the functions $(\ftmu,\hat{\eta})$ on a finite lattice or mesh (i.e., $\mathsf{L}^{(D)}$), so each SGLD iterate is a lattice-based approximation of the continuous predictor functions.  
Sampling and optimization are therefore performed in a discretized representation of the BFT, whose precise construction is given in \cref{sec:bft_discretization}.
Posterior samples $\{(\ftmu^{(m)},\hat{\eta}^{(m)})\}_{m=1}^M$ represent discrete function realizations drawn from this approximate posterior.  
Their dispersion across $\ftmu^{(m)}$ captures epistemic uncertainty, while the Monte Carlo mean $\mathbb{E}[e^{-\hat{\eta}^{(m)}}]$ estimates the expected aleatoric variance function.

Rather than drawing full functional samples, we approximate the posterior by training an ensemble of independently initialized FT models, each converging to a distinct local MAP solution.
The variability across the ensemble provides a Monte Carlo approximation to the Bayesian posterior 
$p(\ftmu,\hat{\eta}\sep\mathcal{D})$.
This ensemble view operationalizes the BFT: the deterministic FT gives the MAP equations, while multiple FT realizations collectively capture epistemic uncertainty through their dispersion in predictive means and noise functions.

The behavior of the BFT can be visualized by overlaying these independent FT realizations, which reveal the spread of predicted means and variances across ensemble members (\cref{fig:multi-ens}). 
Rather than integrating out epistemic variation, we directly display the spread of predicted means $\ftmu^{(m)}(x)$ and variances $e^{-\hat{\eta}^{(m)}(x)}$ across ensemble members.
This representation highlights the posterior support and qualitative variability of solutions, 
which together approximate the epistemic uncertainty of the BFT.

\subsubsection{Discretization and practical approximation of the continuum BFT}
\label{sec:bft_discretization}

As in the deterministic FT (\cref{sec:deterministic_discretization}), the continuous predictor fields $(\ftmu,\hat{\eta})$ are evaluated on a finite lattice $\mathsf{L}^{(D)}=\{x_i^{(D)}\}_{i=1}^D \subset \mathcal{X}$ to obtain discrete representations $(\ftmu^{(D)},\hat{\eta}^{(D)})$.  
We assume that the empirical measure of the lattice converges weakly to the data distribution,
\[
    \frac{1}{D}\sum_{i=1}^D \delta_{x_i^{(D)}} \;\Rightarrow\; p(x)\,dx,
\]
so that weighted sums over $\mathsf{L}^{(D)}$ provide Monte--Carlo approximations to $p$--integrals.

In principle, one obtains finite-dimensional analogues of the continuous Gaussian field priors in \cref{eq:bft_mu_prior_main,eq:bft_eta_prior_main} by replacing spatial derivatives with centered finite-difference operators and integrals with weighted sums $\sum_i p(x_i^{(D)})\,(\cdot)$.  
As in the deterministic discretization, one ghost node is added on each side of the domain and reflected to enforce homogeneous Neumann boundary conditions, allowing centered differences to be used at every interior lattice point.  
The resulting discrete priors are Gaussian Markov random fields (GMRFs): multivariate Gaussian distributions with sparse precision matrices encoding local conditional independences on the lattice.  
These precision matrices provide consistent finite-difference approximations of the weighted elliptic operators associated with the continuous priors, including the weighted Laplacian $\mathcal{L}_p f = -\nabla\!\cdot(p\nabla f)$.

Under standard assumptions on regularity, boundary conditions, and mesh refinement, such GMRF priors converge weakly, in the sense of finite-dimensional distributions, to their continuous Gaussian field counterparts as $D \to \infty$ \citep{lindgren_explicit_2011,lindgren_spde_2022}.  
Further details on the discrete operators and their continuum limits are provided in Appendix~\cref{app:bft_discrete_appendix}.  
Thus the lattice-based construction above yields a principled finite-dimensional approximation of the continuum Bayesian field theory.

In practice, however, the numerical implementation used in \cref{sec:experiments} adopts a simpler and more computationally tractable approximation.  
Rather than sampling from or explicitly forming the GMRF prior, we solve an ensemble of deterministic field theory optimizations, each initialized with a different random seed.  
Each run produces a solution of the deterministic FT on the lattice grid, and the resulting ensemble of solutions provides an empirical approximation to the posterior variability that would be induced by the full BFT model.  
This ensemble-based procedure captures meaningful epistemic variability while avoiding the computational cost of full GMRF-based inference.

\subsection{Predictive Decomposition}
\label{sec:bft_predictive}

Because both latent functions $(\ftmu,\hat{\eta})$ are random under the posterior, the BFT admits a hierarchical uncertainty decomposition that distinguishes variability in the predictive mean from variability in the noise function.
This hierarchy allows the model to represent not only uncertainty about predictions, but also uncertainty about how predictable each region of the input space is.
Formally, the Bayesian predictive distribution marginalizes over both functions,
\begin{equation}
p(y\sep x,\mathcal D)
 =\iint p(y\sep x,\ftmu,\hat{\eta})\,p(\ftmu,\hat{\eta}\sep\mathcal D)\,d\ftmu\,d\hat{\eta},
\end{equation}
with predictive variance
\begin{equation}
\mathrm{Var}[y\sep x,\mathcal D]
 =\mathbb{E}[e^{-\hat{\eta}(x)}\sep\mathcal D]
  +\mathrm{Var}[\ftmu(x)\sep\mathcal D],
\end{equation}
where the first term integrates over uncertainty in the noise function (aleatoric) and the second reflects uncertainty in the mean function (epistemic).
This decomposition holds generally for any joint posterior $p(\ftmu,\hat{\eta}\sep\mathcal D)$ and does not depend on the specific choice of priors, so long as $\ftmu$ and $\hat{\eta}$ are treated as random functions with finite second moments.

Although the predictive mean depends directly on $\ftmu(x)$, we write the expectation $\mathbb{E}_{\ftmu,\hat{\eta}}[\ftmu(x)]$ since the posterior $p(\ftmu,\hat{\eta}\sep\mathcal D)$ couples both functions through the likelihood. 
In practice, this reduces to a marginal expectation over $\ftmu$ once $\hat{\eta}$ is integrated out.

These expectations can be estimated from a finite ensemble of MAP or SGLD function realizations. 
Let $\{(\ftmu^{(m)}, \hat{\eta}^{(m)})\}_{m=1}^M$ denote $M$ independent samples.
Then the pointwise predictive mean and uncertainty components can be approximated as
\[
\ftmu^*(x) = \tfrac{1}{M}\sum_{m=1}^{M}\ftmu^{(m)}(x),\qquad
\hat{\sigma}_{\text{epi}}^2(x) = \mathrm{Var}_m[\ftmu^{(m)}(x)],\qquad
\hat{\sigma}_{\text{ale}}^2(x) = \tfrac{1}{M}\sum_{m=1}^{M} e^{-\hat{\eta}^{(m)}(x)}.
\]
The total predictive uncertainty is given by 
$\hat{\sigma}_{\text{tot}}(x) = (\hat{\sigma}_{\text{epi}}^2(x) + 
\hat{\sigma}_{\text{ale}}^2(x))^{1/2}$.
These Monte Carlo estimators provide a practical implementation of the Bayesian predictive decomposition, linking the theoretical formulation directly to ensemble-based experiments.

\subsection{Statistical Interpretation and Operator Geometry}
\label{sec:bft_statistical_interpretation}

The Bayesian FT induces Gaussian field priors whose precision operator is the weighted Laplacian $\mathcal L_p f = -\nabla\!\cdot(p\nabla f)$, as introduced in \cref{lem:weighted_greens}.  
This places the model squarely within the frameworks of operator–based Gaussian processes and reproducing–kernel Hilbert spaces (RKHS) \citep{kimeldorf_wahba_bayes_1970, Wahba90, lindgren_explicit_2011}. 
In this view, the Dirichlet energy $\int p\,\|\nabla f\|^2$ plays the role of a smoothness penalty: in the GP formulation it arises from the covariance kernel $(\mathcal L_p)^{-1}$, while in the field–theoretic formulation it appears as the corresponding precision operator. 
These two descriptions are mathematically equivalent and provide a unified interpretation of the FT as a Gaussian prior over functions with
weighted Sobolev structure.

This viewpoint clarifies the role of the additive FT geometry. 
Penalizing $\|\nabla \ftp\|^2$ enforces smoothness in the absolute noise level, while the log–precision parameterization $\hat\eta = \log \ftp$ ensures positivity and preserves the form of the Gaussian likelihood. 
When $\ftp$ varies smoothly and is bounded away from zero, these geometries differ only by a spatially varying rescaling through $\nabla\hat{\eta} = \nabla\ftp / \ftp$, yielding similar stationary behavior within the stable region of the regularization space.
The BFT thus recovers the deterministic FT at the posterior mode while also providing a principled probabilistic interpretation of the regularizers and the associated uncertainty decomposition.

This operator-based interpretation connects the FT to classical statistical treatments of heteroskedastic regression, in which smoothness penalties on mean and variance functions arise from Gaussian priors \citep{yuan_doubly_2004,lemm_bayesian_2000}. 
The Bayesian formulation adopted here provides the continuum analogue: a Gaussian prior defined through a differential operator, yielding a flexible nonparametric model whose posterior mode corresponds exactly to the deterministic FT and whose ensemble-based approximations provide a practical route to epistemic uncertainty estimation.

\begin{figure}[htbp]
\centering

\begin{subfigure}[t]{0.495\textwidth}
    \centering
    \includegraphics[width=\textwidth]{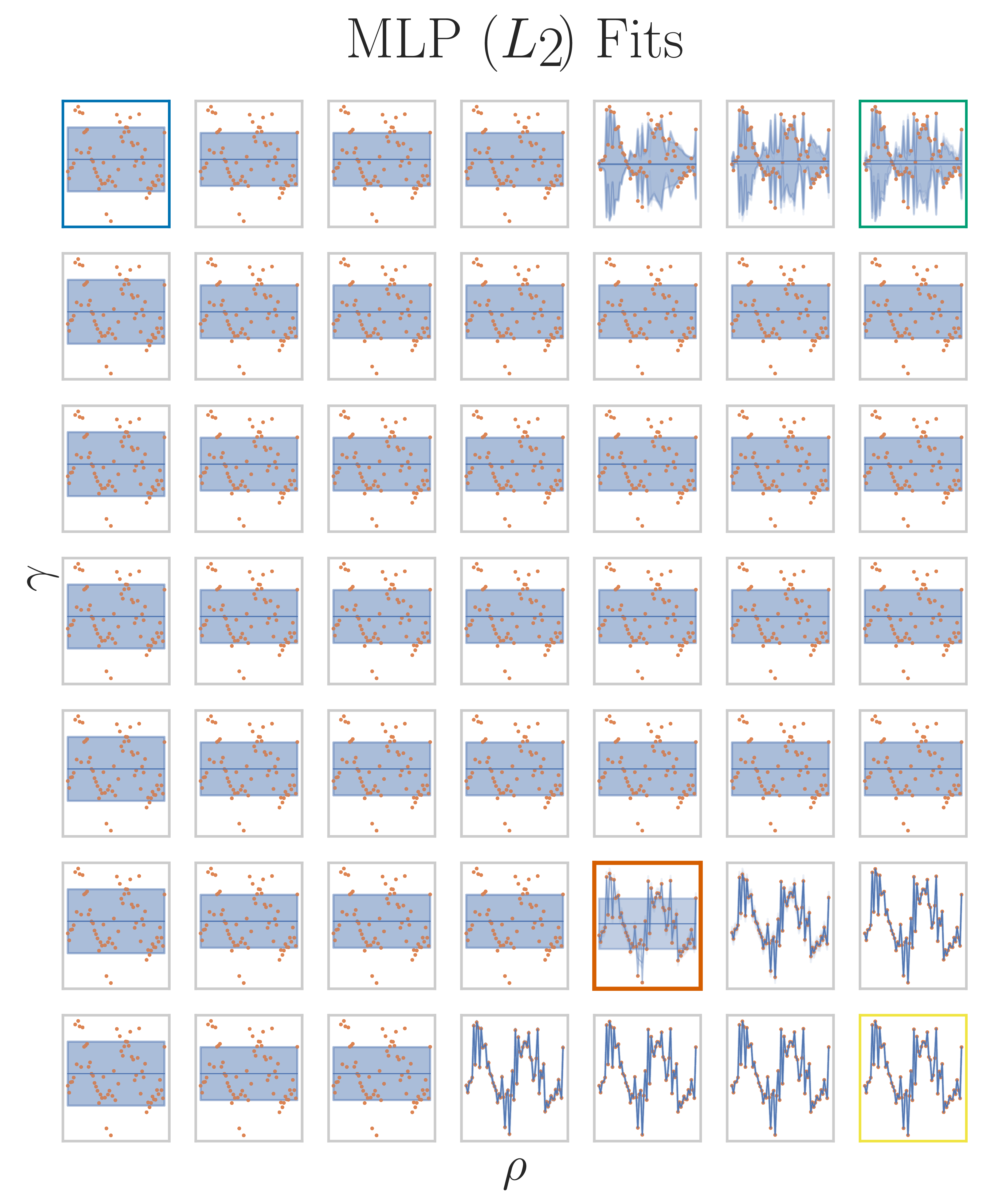}
    \caption{}
    \label{fig:ensemble_mlp_l2}
\end{subfigure}
\hfill
\begin{subfigure}[t]{0.495\textwidth}
    \centering
    \includegraphics[width=\textwidth]{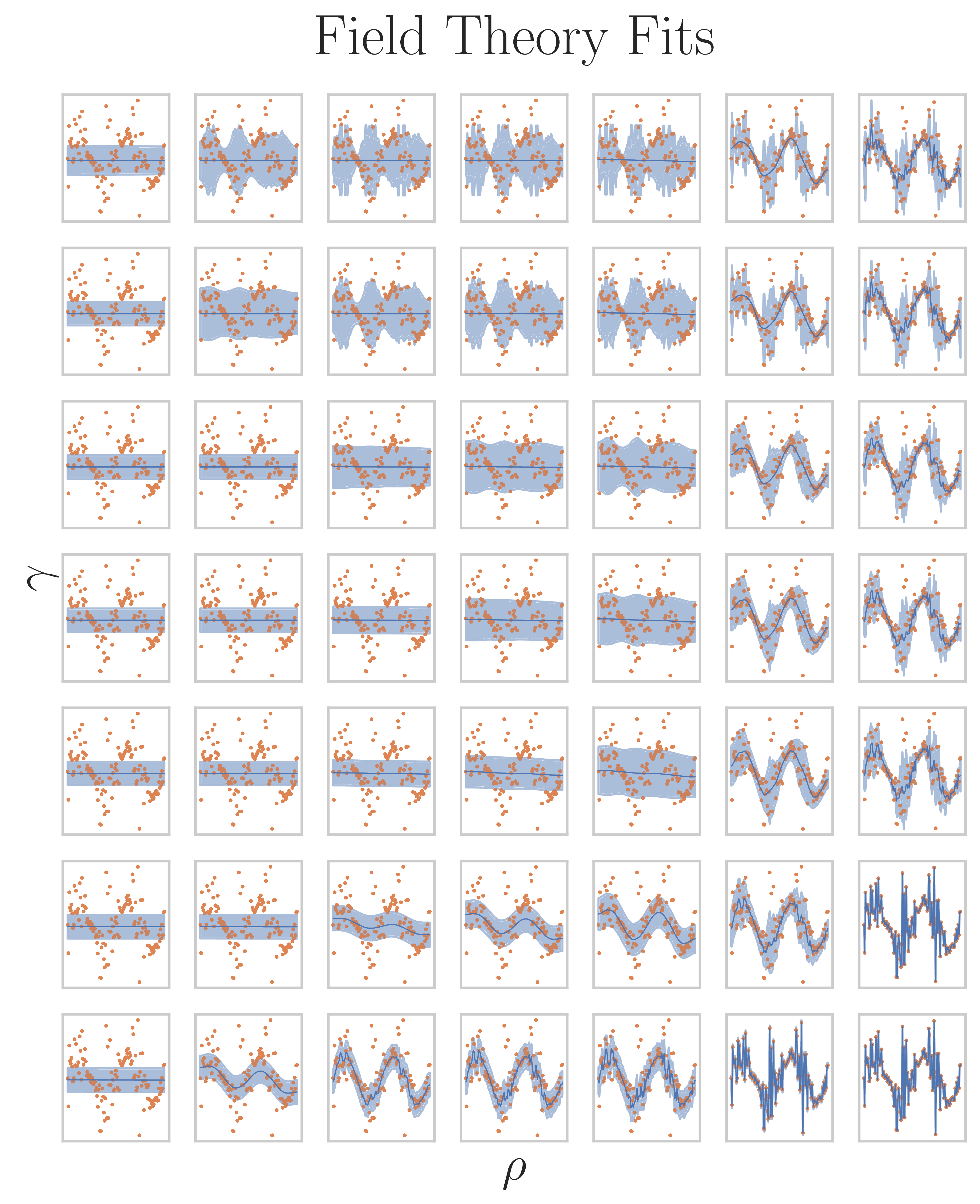}
\caption{}
    \label{fig:ensemble_ft}
\end{subfigure}

\vspace{1em}

\begin{subfigure}[t]{0.21\textwidth}
    \centering
    \includegraphics[width=\textwidth]{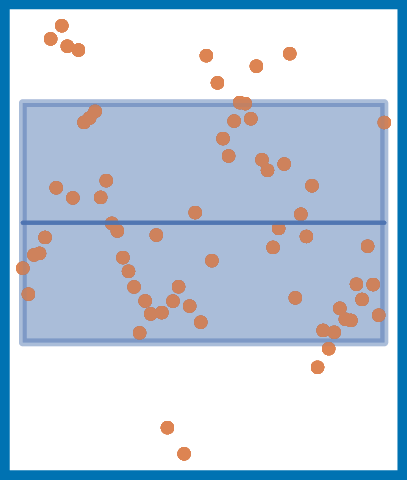}
    \caption{}
    \label{fig:ensemble_c}
\end{subfigure}
\hfill
\begin{subfigure}[t]{0.21\textwidth}
    \centering
    \includegraphics[width=\textwidth]{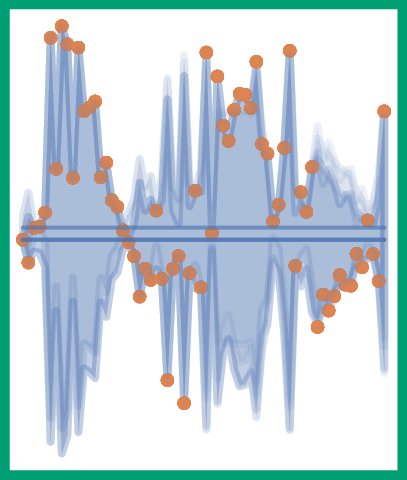}
    \caption{}
    \label{fig:ensemble_d}
\end{subfigure}
\hfill
\begin{subfigure}[t]{0.21\textwidth}
    \centering
    \includegraphics[width=\textwidth]{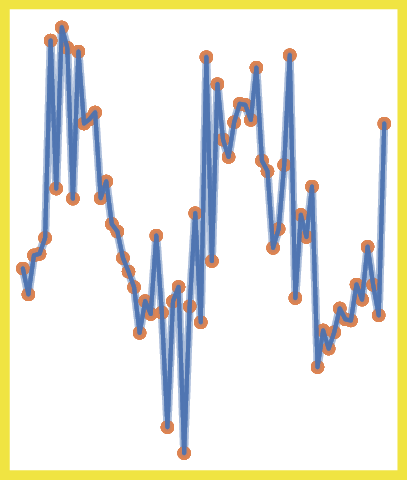}
    \caption{}
    \label{fig:ensemble_e}
\end{subfigure}
\hfill
\begin{subfigure}[t]{0.21\textwidth}
    \centering
    \includegraphics[width=\textwidth]{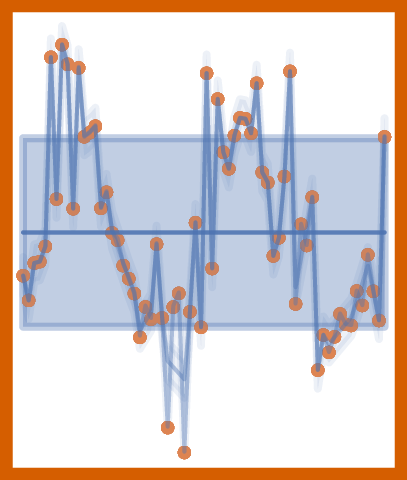}
    \caption{}
    \label{fig:ensemble_f}
\end{subfigure}

\caption{
Ensemble fits from two modeling approaches. 
Training data are shown in orange; the ensemble mean (blue) and its pointwise $\pm 1$\,s.d.\ band (shaded) are overlaid for six independent runs. 
Panels~(\subref{fig:ensemble_mlp_l2}) and (\subref{fig:ensemble_ft}) show a neural implementation and its FT counterpart, respectively. 
Panels~(\subref{fig:ensemble_c})--(\subref{fig:ensemble_f}) illustrate representative neural network fits in different overfitting and underfitting regimes, with panel~(\subref{fig:ensemble_f}) displaying phase coexistence in $(\rho,\gamma)$ space.
}

\label{fig:multi-ens}
\end{figure}

\section{Experiments}
\label{sec:experiments}

The primary goal of our experiments is to visualize phase transitions in two-dimensional phase diagrams.
We show that the qualitative structure of these phase diagrams is independent of any specific neural network architecture by demonstrating close agreement with the field-theoretic solutions.
This analysis also yields a practical procedure for selecting well-suited $(\rho,\gamma)$ regularization strengths, reducing a two-dimensional hyperparameter search to one dimension.
Our main experiments use fully connected networks with $(\rho,\gamma)$–$L_2$ regularization, and we additionally assess the variability of model fits across multiple runs.

\subsection{Field Theory and Neural Networks}
\label{sec:ft2nn}

The field-theoretic formulation developed in \cref{sec:field_theory} describes the behavior of overparameterized mean--variance regression models in a  function-space limit, abstracting away architectural details of specific predictors.  
In practice, however, we implement these ideas using fully connected neural networks.  
Although neural networks do not span the same function space as the nonparametric fields  considered by the FT, modern overparameterized architectures are sufficiently expressive  to approximate the relevant solution classes and to exhibit the same characteristic  phase transitions predicted by the theory.

Our goal is therefore not to enforce an exact architectural correspondence, but to verify  empirically that standard neural networks trained with simple $(\rho,\gamma)$--weighted  $L_2$ regularization follow the qualitative regimes identified by the FT.  
We apply these penalties separately to the mean and precision networks, mirroring the allocation of smoothness in the field theory, and we observe sharp transitions between underfitting, stable, and overfitting regimes across datasets.  
The resulting neural-network phase diagrams closely match those obtained from numerically solving the FT, demonstrating that the field theory captures the coarse-grained behavior of practical heteroskedastic regressors without requiring architectural modifications.

\subsection{Modeling Choices} 
We chose $\nnmu, \nnLambda$ to be fully-connected networks with three hidden layers of 128 nodes and leaky ReLU activation functions.  
The first half of training was only spent on fitting $\nnmu$, while in the second half of training, both $\nnmu$ and $\nnLambda$ were jointly learned. 
This improves stability, since the precision is a dependent on the mean $\nnmu$, and is similar in spirit to ideas presented in \citet{skafte_reliable_2019}. 
Complete training details can be found in Appendix~\ref{app:training}. 

\subsection{Datasets}
We study regularization effects on several one-dimensional simulated datasets and on standardized versions of the \emph{Concrete} \citep{yeh_i-cheng_concrete_2007}, \emph{Housing} \citep{harrison_hedonic_1978}, \emph{Power} \citep{tufekci_prediction_2014}, and \emph{Yacht} \citep{gerritsma_geometry_1981} regression datasets from the UCI 
Repository~\citep{kelly_uci_nodate}, along with a scalar variable from the ClimSim dataset \citep{yu_climsim_2023}. 
We fit neural networks to both simulated and real data, and additionally solve the FT on the simulated datasets.  
Dataset details appear in Appendix~\ref{app:datasets}.  
We show results for the \emph{Sine} dataset and the four UCI datasets, with additional simulated results in Appendix~\ref{app:nn-training}.

\begin{figure}[htbp]
\centering
\includegraphics[width=\textwidth]{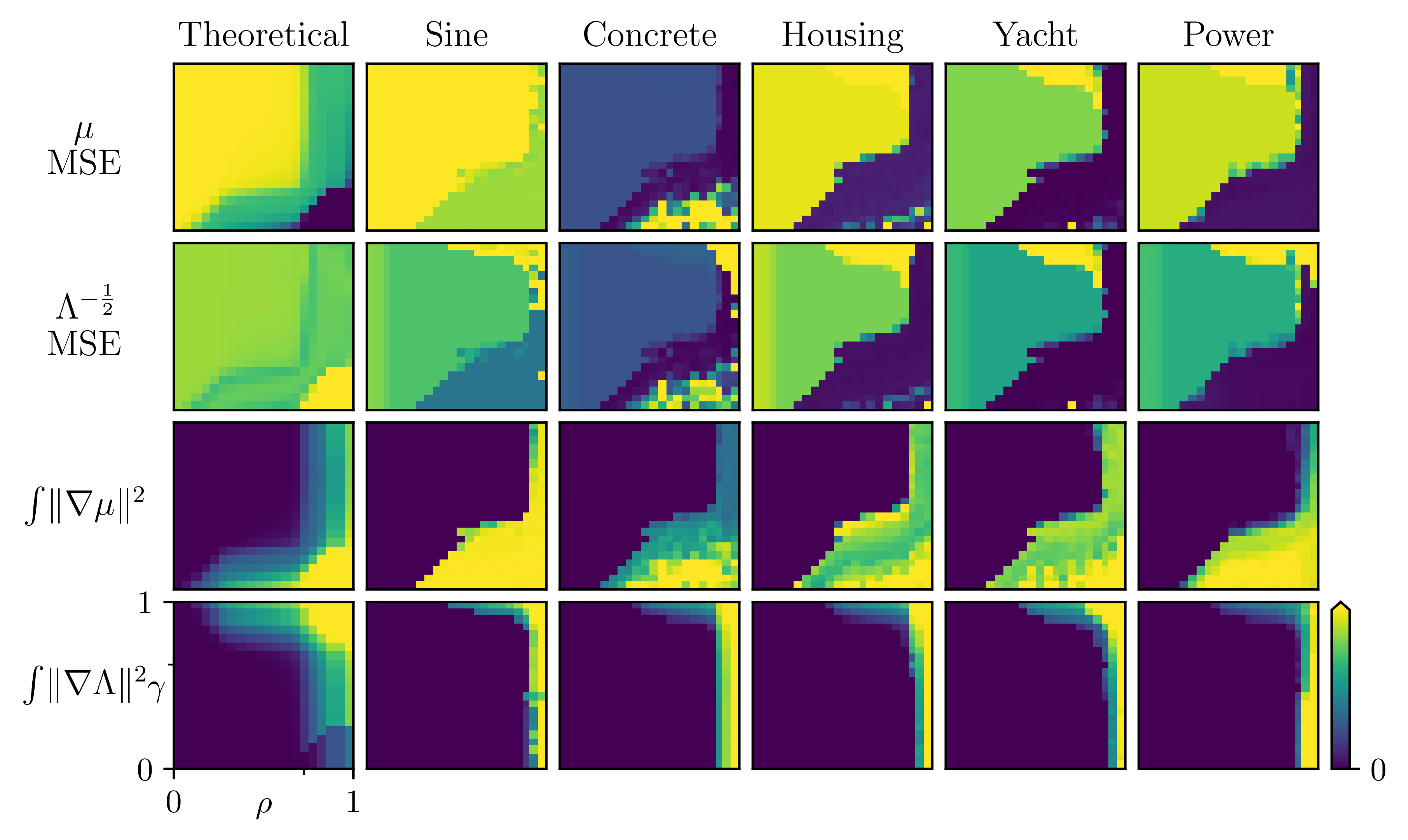}
\caption{
Array plot of evaluation metrics (\emph{rows}) across datasets or fitting methods (\emph{columns}) on the $(\rho,\gamma)$ regularization grid. 
The \emph{leftmost} column shows FT solutions; remaining columns show neural-network fits on held-out data. 
Each heatmap averages six runs. Ticks mark $\rho=0.5$ and $\gamma=0.5$ in the lower-left panel. 
Both axes use a logit parameterization of $\rho,\gamma\in(0,1)$ to highlight limiting behaviors near $0$ and $1$. 
The FT captures the same transition structure observed in the empirical diagrams across datasets.
}
\label{fig:summary_mean}
\end{figure}

\subsection{Qualitative Analysis}

Our qualitative analysis aims at understanding architecture-independent aspects of mean-variance regression upon varying the regularization strength on the mean and variance functions,  resulting in the observation of phase transitions. 

\subsubsection{Observables}
\label{sec:observables}
We evaluate both the calibration and expressiveness of the learned models.
For calibration, we compute mean squared error (MSE) for the predicted mean, $\nnmu(x_i)$, and for the predicted standard deviation, $\Lambda^{-1/2}(x_i)$, comparing the latter to the absolute residuals $|y_i - \nnmu(x_i)|$.
Well-fit models exhibit low errors for both measures, and we report $\Lambda^{-1/2}$-MSE due to its connection to variance-calibration metrics \citep{skafte_reliable_2019, levi_evaluating_2022}.

To assess expressiveness, we compute the Dirichlet energy for the FT solutions and its discrete analogue, the geometric complexity \citep{dherin_why_2022}, for neural networks.
The Dirichlet energy of a function $f$ is $\int_{\mathcal X} p(x)|\nabla f(x)|_2^2dx$, while geometric complexity is $N^{-1}\sum_{i=1}^N |\nabla f(x_i)|_2^2$.
Both quantify the variability of the learned functions, with larger values indicating more expressive (less regularized) behavior and directly corresponding to the quantities penalized in the FT formulation.

\subsubsection{Plot Interpretation}
We present summaries of the fitted models in grids with $\rho$ on the $x$-axis and $\gamma$ on the $y$-axis in \cref{fig:summary_mean,fig:summary_sd}. 
The far right column ($\rho=1$) corresponds to MLE solutions. 
The main focus is on qualitative traits of fits under different levels of regularization and how they behave in a relative sense, rather than a focus on absolute values. 
\cref{fig:diag_slice_a} show the summary statistics along the slice where $\rho = 1-\gamma$. 
Zero on these plots corresponds to the upper left corner while one corresponds to the lower right corner. 
We provide model fits arranged in grids of the same orientation for the field theory and neural networks on the \emph{Sine} dataset in \cref{fig:multi-ens}.

\begin{figure}[htbp]
\centering
\includegraphics[width=\textwidth]{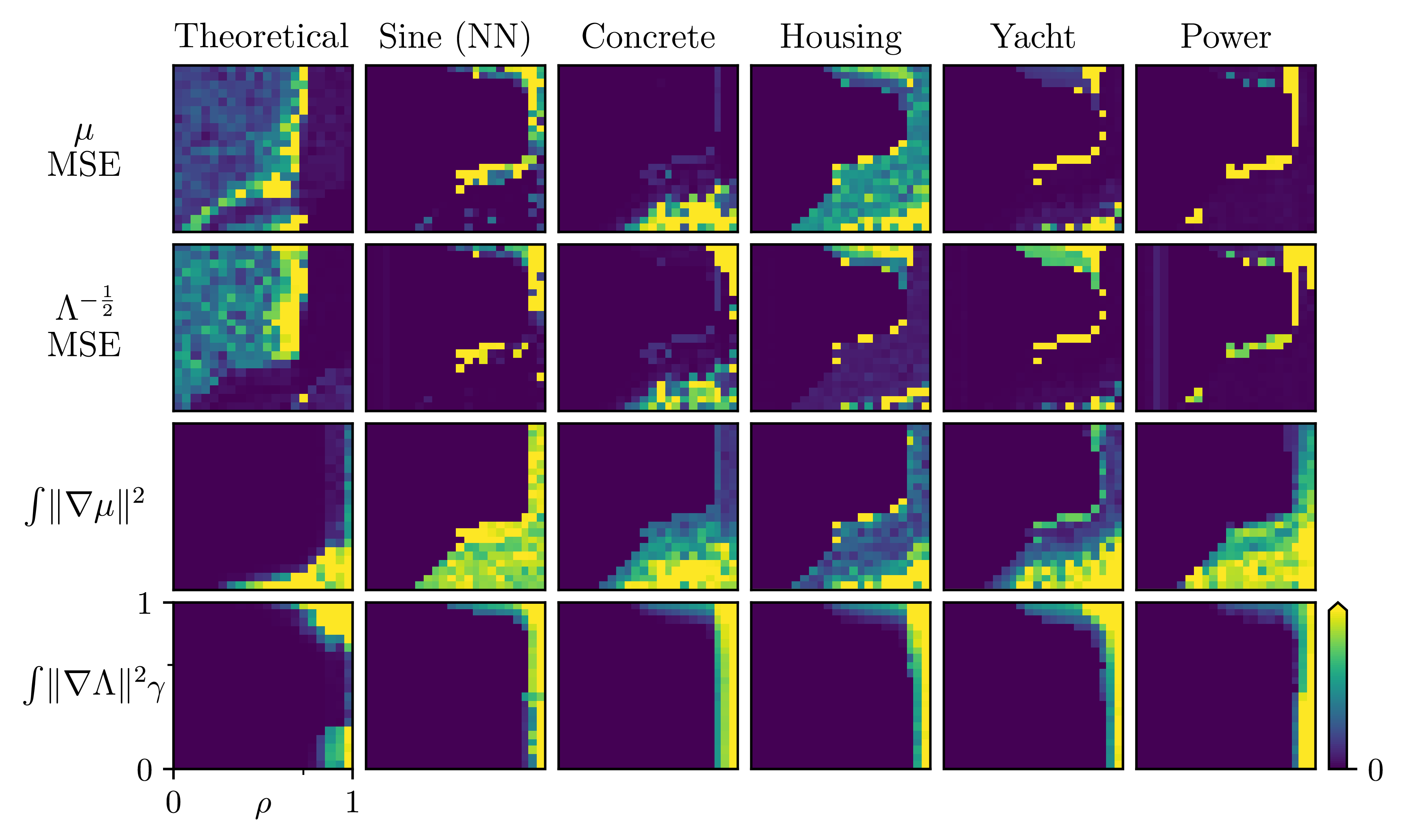}
\caption{
We compute the standard deviation over six runs for each metric in \cref{fig:summary_mean}, illustrating how variability changes across the regularization space. 
The shapes of the instability regions remain consistent across datasets and between the neural networks and the FT, as reflected in the Dirichlet energies and geometric complexities. 
These quantities show the largest disagreement in overfitting regimes, though this does not always correspond to high variability in the MSEs.
}

\label{fig:summary_sd}
\end{figure}

\subsubsection{Variability Over Runs}\label{sec:mvr-variability}
The preceding experiments focused on pointwise estimation of the mean and variance functions, capturing a single prediction and its associated \emph{aleatoric} uncertainty. 
To assess \emph{epistemic} uncertainty, we examine the variability of phase diagrams across multiple independent MVR fits. 
For the one-dimensional synthetic datasets, we visualize this variability directly by plotting individual model fits, illustrating how the consistency of learned functions changes across regions of the phase diagram.

\begin{obs}
Our metrics show sharp phase transitions upon varying $\rho, \gamma$, as in a physical system.
\end{obs}
\cref{fig:summary_mean} and \cref{fig:diag_slice_a} show a sharp transition, both leading to worsening and improving performance when moving along the minor diagonal. 
In totality, across all metrics, the five regions are apparent.
But not all of the regions in \cref{fig:cartoonphases} appear in the heatmaps of each metric. 
For example, region $O_\Lambda$ does not always appear in the metrics related to the mean.  
When using neural networks to approximate $\mu$ and $\Lambda$, there are sharper boundaries between phases than in the FT's numerical solutions. 
The boundary between $U_\mu$ and $O_\mu$ is sharply observed in the plots of $\int \|\nabla \mu(x)\|_2^2 \, dx$. 
However, in terms of $\mu$-MSE, a smoother transition (i.e., region $S$) is visible.

\begin{obs}
    The FT insights and observed phases are consistent with both the numerically solved FT and the neural-network fits. 
    Thus, our conclusions are not tied to  a specific architecture or dataset.
\end{obs}

In line with the theoretical predictions, phases $U_{\Lambda}$ and $O_{\mu}$ exhibit consistent behavior across $\gamma$-values (vertical slices in \cref{fig:summary_mean}).  Across all datasets we considered, the qualitative structure  of the phase diagrams remains the same: the same phase types appear, the same ordering of regimes is observed, and the transitions occur along similarly shaped boundaries.  
Representative fitted models are shown in \cref{fig:multi-ens}.

While the overall phase structure is shared across datasets, the precise locations of the transitions vary.  
Different datasets and input dimensionalities shift the $(\rho,\gamma)$ values at which the boundaries occur, but the geometric shape and ordering of regions remain stable.

In the right-hand columns $(\rho \!\to\! 1)$ the mean function nearly interpolates the data, and similar behavior is seen in the lower rows $(\gamma \!\to\! 0)$.  
Across all metrics, the regions evolve with regularization strength in a comparable manner on all datasets.  Notably, the region of small $\int \|\nabla \Lambda(x)\|_2^2 \, dx$ covers a larger portion of the diagram than the corresponding region for $\int \|\nabla \mu(x)\|_2^2 \, dx$, indicating that the precision function remains smoother than the mean under comparable levels of regularization.

\begin{obs}\label{obs:variability}
The neural network phase diagrams reveal different amounts of variability in model fits across the regularization space.
\end{obs}

This behavior hints at variability in the fitting procedure and can be considered a sign of the need to measure epistemic uncertainty. 
The standard deviations over the metrics displayed in \cref{fig:summary_mean} are shown in \cref{fig:summary_sd}. 
The Dirichlet energies/geometric complexities show that there is the most variability in the overfitting regions $O_\mu$ and parts of $O_\Lambda$. 
This indicates that the functions themselves vary across runs. 
Actual fits of the \emph{Sine} dataset are displayed in \cref{fig:multi-ens}. 
However, when turning to quality of fits, the MSEs show a different pattern of regions of instability, and $O_\mu$ has low variability in terms of actual performance.

Note that in the region of high regularization (\emph{far left column}) we see greater variability than in the moderately regularized regions in the upper middle. 
We posit that in the highly regularized regions we are essentially seeing the variability coming from the random initialization of the model weights. 
Meanwhile in the central region we see that the mean and variance functions are afforded enough ``flexibility'' to adapt to the global mean and standard deviation, but not enough flexibility to fit to the data. 
Thus there is much less spread between the different ensemble members in this region. 

\begin{obs}\label{obs:instability}
The neural network phase diagrams exhibit regions of instability: for certain $(\rho,\gamma)$ values, independently trained MLPs produce qualitatively different fits, whereas the FT solutions are highly consistent across runs.
\end{obs}

\cref{fig:multi-ens} illustrates this phenomenon.  
Even with identical $(\rho,\gamma)$ values and full-batch gradient descent, different random initializations of the MLP parameters can lead to distinct fitting behaviors---some runs overfit while others underfit the data (see the outlying curve in
\cref{fig:ensemble_mlp_l2}).
This variability is quantified in \cref{fig:summary_sd}, which reports the standard deviation of MSEs and Dirichlet energies across runs and reveals pockets of substantial instability in the neural-network landscape.

In contrast, the field-theoretic fits (\cref{fig:ensemble_ft}) show almost no variation across runs. 
Although we also initialize the discretized fields $\mu$ and $\Lambda$ randomly and optimize them by gradient descent, the FT energy appears to have a much smoother and more strongly regularized landscape: empirically, all runs converge to essentially the same solution for a fixed $(\rho,\gamma)$. 
We do not claim uniqueness of the FT minimizer analytically, but in practice the FT optimization exhibits a single stable attractor, in sharp contrast to the multiple effective basins observed for the neural networks.

\begin{figure}[htbp]
\centering
\begin{subfigure}{\textwidth}
   \includegraphics[width=\linewidth]{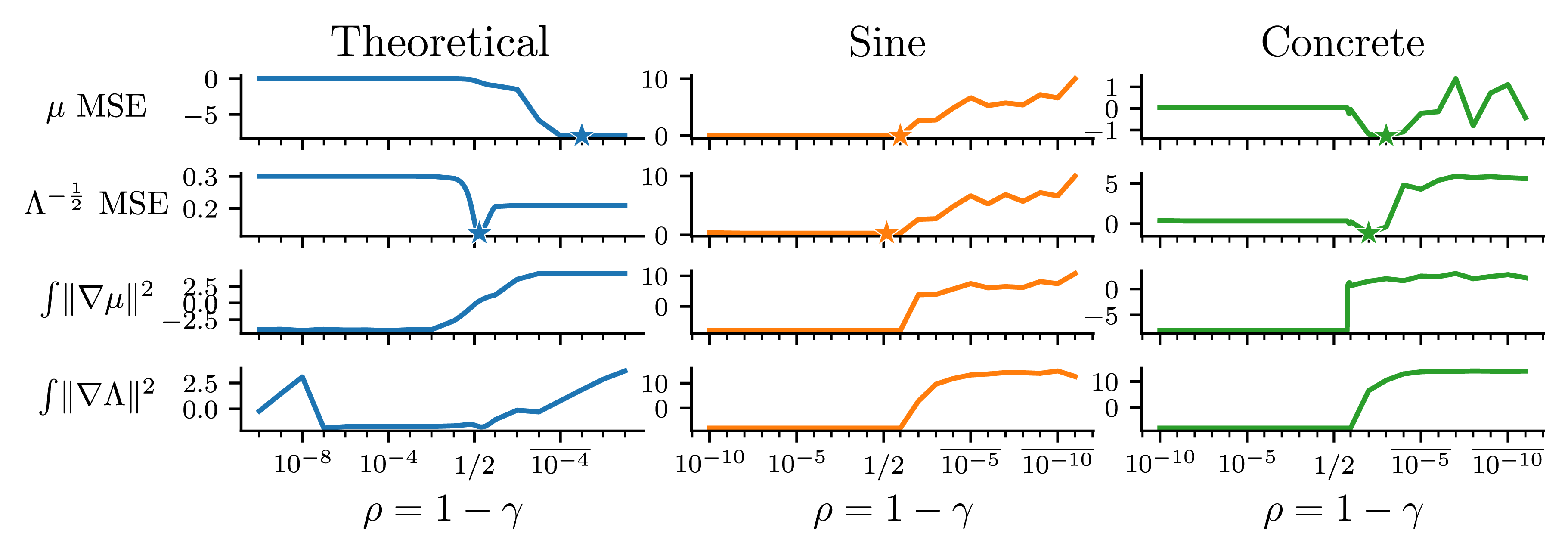}
\end{subfigure}
\begin{subfigure}{\textwidth}
   \includegraphics[width=\linewidth]{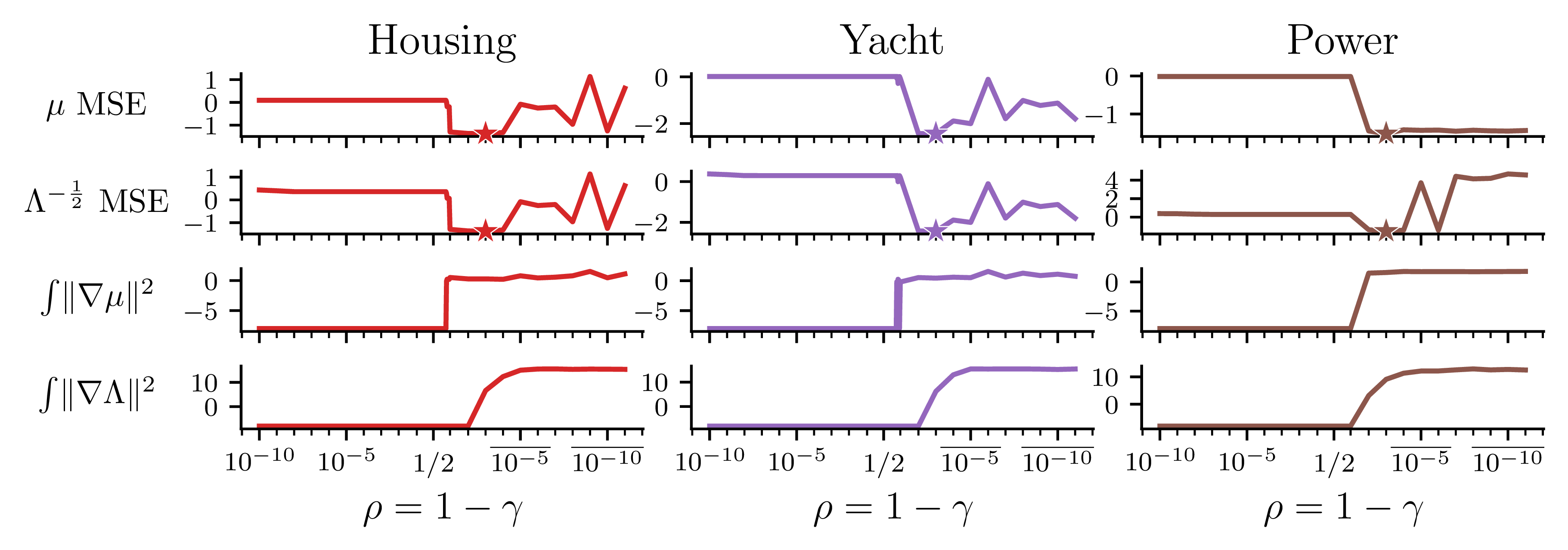}
\end{subfigure}

\caption{
Test metrics across six runs along the $\rho = 1-\gamma$ diagonal. 
Stars denote the minimum MSE for each dataset. 
All metrics are plotted on a $\log_{10}$ scale, and $\rho$ is shown on a logit scale to highlight behavior near the boundaries. 
Errors drop sharply near the transition into the $S$ phase and then increase again as $\rho$ moves past this region, consistent with the qualitative structure in \cref{fig:cartoonphases}.
}

\label{fig:diag_slice_a}
\end{figure}
\subsection{Quantitative Analysis}
Our quantitative analysis aims to demonstrate the practical implications of our qualitative investigations that result in better calibration properties.

\begin{obs}
We can search along $\rho = 1-\gamma$ to find a well-calibrated $(\rho, \gamma)$-pair from region $S$.
\end{obs}

\begin{table*}[htp!]
\centering
\small
\caption{Comparison of our mean-variance regression model (Ours) with diagonal regularization search and $\beta$-NLL \citep{seitzer_pitfalls_2022}.
Details on our MVR selection criteria can be found in Appendix~\ref{app:diag-crit}.
We report the average and standard deviations of $\mu$- and $\Lambda^{-\frac{1}{2}}$-MSE across six runs on test data.}

\begin{tabular}{ @{}l|*{6}{c}@{} }
\toprule
{Metric} &
{Sine} &
{Concrete} &
{Housing} &
{Power} &
{Yacht} &
{Solar Flux} \\
\midrule
$\mu$-MSE & & & & & & \\
\hfill Ours  & 0.80 ± 0.00 & \textbf{0.11} ± 0.02 & 1.22 ± 0.00 & \textbf{0.04} ± 0.01 & \textbf{0.01} ± 0.01 & \textbf{0.29} ± 0.00 \\
\hfill $\beta$-NLL    & \textbf{0.69} ± 0.05 & 0.55 ± 0.30 & \textbf{0.32} ± 0.05 & 0.09 ± 0.01 & \textbf{0.01} ± 0.01 & 0.38 ± 0.00 \\
\midrule
$\Lambda^{-\frac{1}{2}}$-MSE & & & & & & \\
\hfill Ours   & 0.80 ± 0.00 & \textbf{0.30} ± 0.51 & \textbf{0.76} ± 0.00 & \textbf{0.03} ± 0.01 & \textbf{0.01} ± 0.01 & \textbf{0.12} ± 0.00 \\
\hfill $\beta$-NLL     & \textbf{0.52} ± 0.07 & 1.09 ± 0.20 & 0.88 ± 0.03 & 0.31 ± 0.37 & 1.33 ± 0.02 & 0.32 ± 0.00 \\
\bottomrule
\end{tabular}
\label{tab:baseline_sub}
\end{table*}

Our FT indicates that a slice across the minor diagonal of the phase diagram should always cross the $S$ region (see \cref{fig:cartoonphases}).
\cref{fig:diag_slice_a} shows that by searching along this diagonal, we indeed find a combination of regularization strengths where both $\nnmu$ and $\nnLambda$ generalize well to held-out test data. 
This implies that there is no need to search all of the two-dimensional space, but only a single slice which reduces the 
number of models to fit from $\bigO(N^2)$ to $\bigO(N)$, where $N$ is the number of $\rho$ and $\gamma$ values that are tested. 
This finding is consistent with the suggestion from \citet{sluijterman_optimal_2024} to have stronger regularization on the variance than the mean. 
This corresponds to searching across a horizontal slice in the lower portion of our phase diagram and is generally consistent with where we posit the well-behaved $S$ region tends to lie. 

\cref{fig:diag_slice_a} shows that along the minor diagonal the performance is initially poor, improves, and then drops off again. 
These shifts from strong to weak performance are sharp. 
The regularization pairings that result in optimal performance with respect to $\mu$- and $\Lambda^{-1/2}$-MSE are near each other along this diagonal for the real-world test data. 
As the theory predicts, the performance becomes highly variable as we approach the MLE solutions and the FT fails to converge in this region. 
In practice, we propose searching along this line to find the $(\rho, \gamma)$-combinations that minimize the $\mu$- and $\Lambda^{-\frac{1}{2}}$-MSEs and averaging the regularization strengths to fit a model.  
We compare models chosen by our diagonal line search to two heteroskedastic modeling baselines in Appendix~\ref{app:baseline_comp} on the synthetic and UCI datasets as well as a scalar quantity from the ClimSim dataset \citep{yu_climsim_2023}. 
We present a subset of the results below in Table \ref{tab:baseline_sub}. 
In most cases the model chosen via the diagonal line search was competitive or better than the baselines.

\section{Conclusion}

Neural mean--variance models are known to exhibit significant training instabilities 
\citep{nix_estimating_1994,nix_learning_1994}. 
By developing a field-theoretic perspective grounded in statistical physics, we derived a 
continuum variational formulation of the learning problem that isolates structural causes of 
these pathologies \citep{lemm_bayesian_2000,ringel_applications_2025}. 
We refer to this continuum formulation as the field theory, and it yields explicit, 
architecture-independent insights into the behavior of deep heteroskedastic regression. 
It also helps clarify why these models often require carefully tuned regularization and why 
they tend to exhibit transitions between qualitatively different regimes.

Building on this continuum formulation, we introduced a numerical discretization of the field 
theory and demonstrated close qualitative agreement with neural network solutions across both 
synthetic and real-world datasets. 
Across repeated fits, we observed two central challenges: inconsistent behavior across model 
runs and inconsistent behavior across regularization strengths. 
The field theory clarifies these effects and motivates a more principled strategy for tuning 
regularization, reducing a two-dimensional search to an effectively one-dimensional problem. 
We also found that independently trained neural networks often undergo their transitions at 
different points in the regularization space. 
This produces an effect reminiscent of phase coexistence, where ensemble variability captures 
structure that is not visible in any single fit. 
The Bayesian field-theoretic perspective clarifies part of this connection to epistemic 
uncertainty, although a full characterization within a complete Bayesian framework remains to 
be developed. 
The Bayesian formulation also links the continuum model to classical statistical frameworks, 
including Gaussian process priors, spline-based smoothing, and penalized likelihood methods 
\citep{Wahba90,lindgren_explicit_2011}, which arise as special cases under particular choices 
of regularizers.

\subsection{Limitations}

The field-theoretic formulation captures several important aspects of neural behavior, but it 
also has limitations. 
By replacing the discrete neural objective with a population-level variational problem, it 
abstracts away optimization dynamics, architectural nonlinearities, and other discrete effects 
that are present in neural networks and may influence the sharpness of observed transitions. 
As a result, the continuum phase diagram should not be expected to match neural behavior in 
detail. 
Nevertheless, across datasets with different scales, dimensionalities, and smoothness levels, 
we observe similar qualitative transition structures, and the field theory reproduces this 
high-level organization. 
However, it does not fully capture quantitative differences such as the severity or abruptness 
of certain transitions seen in neural networks. 
Clarifying how the discrete neural parameterization relates to its continuum limit may help 
close this gap.

Finally, our approach illustrates a broader methodological trade-off. 
By passing to a continuum and imposing a structured variational framework, the field theory 
sacrifices some of the specificity of discrete neural architectures in exchange for analytic 
clarity. 
At the same time, it highlights structural mechanisms that are difficult to isolate directly, 
such as the competing effects of the likelihood and the regularizers, the influence of the 
data density, and the organization of solutions into distinct qualitative regimes. 
Related phenomena arise in other learning problems where objectives exhibit instabilities or 
undergo symmetry breaking, such as the continuous-symmetry scenarios studied by 
\citet{bamler_improving_2018}. 
This reflects a theme familiar from the philosophy of mathematical physics, where idealized 
formulations sacrifice some realism but provide conceptual clarity about the forces that shape 
a system's behavior \citep{sep-qt-nvd}. 
In this sense, the field-theoretic formulation should be viewed as a complementary perspective 
rather than a replacement for the neural model.

Overall, we hope that this work encourages further exploration of phase transitions, 
variational principles, and operator-based ideas as tools for understanding the collective 
and nonlinear behavior that arises in modern large-scale deep learning models.

\section*{Acknowledgments}
Eliot Wong-Toi acknowledges support from the Hasso Plattner Research School at UC Irvine. 
Vincent Fortuin was supported by a Branco Weiss Fellowship. 
Stephan Mandt acknowledges support from the IARPA WRIVA program; the National Science 
Foundation (NSF) under the CAREER Award 2047418 and Grants 2003237 and 2007719; the Department 
of Energy, Office of Science under Grant DE-SC0022331; and gifts from Intel, Disney, and 
Qualcomm.

\appendix

\section{Theoretical Details}
\label{app:theoretical_details} 
\textit{The results summarized in \Cref{prop:summary_main} are derived here in full detail.}

\subsection{General Field Theory and Extreme Settings}
\label{app:app_general_ft}

\subsection*{Weighted Green’s Identity and Natural Neumann Data}

\begin{lem}[Weighted Green’s identity]
\label{lem:weighted_greens}
Let $\mathcal{X}\subset\mathbb{R}^d$ be a bounded Lipschitz domain with outward unit normal $\boldsymbol{n}$ on $\partial\mathcal{X}$, and let $p\in C^1(\overline{\mathcal{X}})$ satisfy $p>0$.
For any $f,g\in H^1(\mathcal{X})$,
\begin{align}
    \int_{\mathcal{X}} p\,\nabla f\!\cdot\!\nabla g\,dx
    = -\int_{\mathcal{X}} g\,\nabla\!\cdot\!\big(p\nabla f\big)\,dx
      + \int_{\partial\mathcal{X}} g\,p\,\nabla f\!\cdot\!\boldsymbol{n}\,dS .
    \label{eq:weighted_greens_basic}
\end{align}

Equivalently, define the (negative) weighted Laplacian
\[
    \mathcal L_p f := -\,\nabla\!\cdot\!\big(p\,\nabla f\big),
\]
so that $\mathcal L_p$ is a positive semidefinite, self-adjoint operator on
$H^1(\mathcal X)$ with respect to the weighted inner product $\langle f,g\rangle_p=\int_{\mathcal X} p\, fg\,dx$.
Then \eqref{eq:weighted_greens_basic} becomes
\begin{align}
    \int_{\mathcal{X}} p\,\nabla f\!\cdot\!\nabla g\,dx
    = \int_{\mathcal{X}} p\,(\mathcal{L}f)\,g\,dx
      + \int_{\partial\mathcal{X}} g\,p\,\nabla f\!\cdot\!\boldsymbol{n}\,dS .
    \label{eq:weighted_greens_operator}
\end{align}
\end{lem}

\begin{proof}
Apply the divergence theorem to the vector field $g\,p\,\nabla f$:
\[
    \int_{\mathcal{X}} \nabla\!\cdot\!\big(g\,p\,\nabla f\big)\,dx
      = \int_{\partial\mathcal{X}} g\,p\,\nabla f\!\cdot\!\boldsymbol{n}\,dS .
\]
Expanding the divergence gives
\[
    \nabla\!\cdot\!\big(g\,p\,\nabla f\big)
     = g\,\nabla\!\cdot\!\big(p\nabla f\big)
       + p\,\nabla f\!\cdot\!\nabla g .
\]
Rearranging yields \eqref{eq:weighted_greens_basic}.  
Substituting $\mathcal{L}f = -\nabla\!\cdot(p\nabla f)$ gives \eqref{eq:weighted_greens_operator}.
\end{proof}

\paragraph{Natural Neumann boundary conditions.}
If the co-normal derivative $(p\nabla f)\cdot\boldsymbol{n}$ is required to vanish
on $\partial\mathcal{X}$ (the homogeneous Neumann condition), then the boundary
term in \eqref{eq:weighted_greens_operator} disappears.  
This yields
\[
\int_{\mathcal{X}} p\,\nabla f\cdot\nabla g\,dx
= \int_{\mathcal{X}} p\,(\mathcal L_p f)\,g\,dx,
\]
showing that homogeneous Neumann boundary conditions arise naturally as the
\emph{natural} boundary conditions of the weighted Dirichlet energy
$\int p\,\|\nabla f\|^2$.

\begin{remark}[On the weighted Laplacian]
With the convention 
\[
\mathcal L_p f := -\,\nabla\!\cdot(p\nabla f),
\]
we may expand in Euclidean coordinates as
\[
\mathcal L_p f
= -p\,\Delta f - (\nabla p)\cdot\nabla f.
\]
A normalized form,
\[
\tilde{\mathcal L}_p f := \frac{1}{p}\mathcal L_p f
= -\Delta f - (\nabla\!\log p)\!\cdot\!\nabla f,
\]
makes explicit the drift term induced by the nonuniform weight $p(x)$.  
Both $\mathcal L_p$ and its normalized form encode the weighted Laplace operator
that arises from integration by parts under the measure $p(x)\,dx$.
\end{remark}

\begin{cor}[Natural (Neumann) boundary conditions]
\label{cor:natural_neumann}
If $(p\nabla f)\cdot\boldsymbol{n}=0$ on $\partial\mathcal{X}$, then for all 
$g\in H^1(\mathcal{X})$,
\[
\int_{\mathcal{X}} p\,\nabla f\cdot\nabla g\,dx
= \int_{\mathcal{X}} p\,(\mathcal L_p f)\,g\,dx.
\]
Thus quadratic Dirichlet energies $\int \tfrac{\kappa}{2} p\|\nabla f\|^2\,dx$ 
induce $\mathcal L_p$ as the Euler--Lagrange operator with homogeneous zero-flux 
as the natural boundary condition.
\end{cor}

\subsection*{General Field Theory Formulation}

\begin{prop}[General Field Theory]
\label{prop:generalFT}
Assume that $\mathcal{X} \subset \mathbb{R}^d$ is a bounded Lipschitz domain, 
and that $p \in C^1(\overline{\mathcal{X}})$ is a strictly positive probability 
density on~$\mathcal{X}$.

Let $\ftmu, \ftp \in H^1(\mathcal{X})$ satisfy
\[
    \ftp(x) \ge c > 0 \qquad \text{for a.e. } x \in \mathcal{X},
\]
and assume their weighted Dirichlet energies are finite:
\[
    \int_{\mathcal{X}} p\,\|\nabla \ftmu\|_2^2\,dx < \infty,
    \qquad
    \int_{\mathcal{X}} p\,\|\nabla \ftp\|_2^2\,dx < \infty.
\]

Define the unnormalized weighted Laplacian
\[
    \mathcal L_p f := -\,\nabla\!\cdot\!\big(p\nabla f\big).
\]

Let
\begin{align}
    \mathbf{S}_{\rho,\gamma}[\ftmu,\ftp]
    = \int_{\mathcal{X}} p(x)
    \Big[-\rho \log \hat{p}(y\sep x)
    + \bar{\rho}\!\left(\gamma\|\nabla \ftmu(x)\|_2^2
    + \bar{\gamma}\|\nabla \ftp(x)\|_2^2\right)\!\Big] dx,
\end{align}
where $\hat{p}(y\mid x)=\mathcal{N}(y\mid \ftmu(x),\ftp(x)^{-1})$, 
$\bar\rho=1-\rho$, and $\bar\gamma=1-\gamma$.
Then stationary points satisfy the Euler–Lagrange equations
\begin{subequations}\label{eq:euler_general}
\begin{align}
    \rho\,p(x)\,\ftp(x)\big(\ftmu(x) - y(x)\big)
    &= 2\bar{\rho}\,\gamma\,\mathcal L_p\ftmu(x),
    \label{eq:euler_mu_general}\\[0.5em]
    \frac{\rho}{2}\,p(x)\!\left[\big(\ftmu(x)-y(x)\big)^2
      - \frac{1}{\ftp(x)}\right]
    &= 2\bar{\rho}\,\bar{\gamma}\,\mathcal L_p\ftp(x),
    \label{eq:euler_lambda_general}
\end{align}
\end{subequations}
with homogeneous Neumann (zero-flux) boundary conditions
\[
    p\,\nabla\ftmu\!\cdot\!\boldsymbol{n}
    = p\,\nabla\ftp\!\cdot\!\boldsymbol{n}
    = 0
    \qquad\text{on }\partial\mathcal{X}.
\]
\end{prop}

\begin{proof}
Rewrite the likelihood term using
\[
-\log\hat p(y\mid x)
=\tfrac{1}{2}\ftp(x)\,\hat r(x)^2
-\tfrac{1}{2}\log\ftp(x)
+ \text{const},
\qquad
\hat r(x)=y(x)-\ftmu(x).
\]
Thus
\[
\mathbf{S}_{\rho,\gamma}[\hat\mu,\hat\Lambda]
=\!\int_{\mathcal X}\! p(x)
\Bigg\{
\rho\Big[\tfrac{1}{2}\hat\Lambda\,\hat r^2
-\tfrac{1}{2}\log\hat\Lambda\Big]
+\bar\rho\Big[\gamma\|\nabla\hat\mu\|_2^2
+\bar\gamma\|\nabla\hat\Lambda\|_2^2\Big]
\Bigg\}dx.
\]

For test functions $\varphi,\psi\in H^1(\mathcal X)$, consider perturbations
$\hat\mu_\varepsilon=\hat\mu+\varepsilon\varphi$ and
$\hat\Lambda_\varepsilon=\hat\Lambda+\varepsilon\psi$.  
We compute the first variations.

\paragraph{Variation with respect to $\ftmu$.}
Only the terms $\tfrac{1}{2}\hat\Lambda\hat r^2$ and 
$\gamma\|\nabla\hat\mu\|^2$ depend on~$\hat\mu$.
Since $\hat r=y-\hat\mu$,
\[
\frac{d}{d\varepsilon}
\Big(\tfrac{1}{2}\hat\Lambda\,\hat r_\varepsilon^2\Big)\Big|_{0}
=\hat\Lambda(\hat\mu-y)\,\varphi.
\]
The gradient term gives
\[
\frac{d}{d\varepsilon}\|\nabla\hat\mu_\varepsilon\|^2\Big|_{0}
=2\,\nabla\hat\mu\cdot\nabla\varphi.
\]
Thus
\[
\delta \mathbf{S}[\hat\mu;\varphi]
=\int_{\mathcal X} p\,\rho\,\hat\Lambda(\hat\mu-y)\varphi\,dx
+2\bar\rho\gamma\int_{\mathcal X}p\,\nabla\hat\mu\cdot\nabla\varphi\,dx.
\]

Applying the weighted Green’s identity
\[
\int_{\mathcal X} p\,\nabla u\cdot\nabla v\,dx
= \int_{\mathcal X} v\,\mathcal L_p u\,dx
+ \int_{\partial\mathcal X} p\,v\,\tfrac{\partial u}{\partial n}\,dS,
\]
with $u=\hat\mu$, $v=\varphi$, yields
\begin{align*}
\delta \mathbf{S}[\hat\mu;\varphi]
&=\int_{\mathcal X}
\Big[\,p\rho\,\hat\Lambda(\hat\mu-y)
+2\bar\rho\gamma\,\mathcal L_p\hat\mu\,\Big]\varphi\,dx
+2\bar\rho\gamma
\int_{\partial\mathcal X}p\,\varphi\,\tfrac{\partial\hat\mu}{\partial n}\,dS.
\end{align*}

Stationarity for all $\varphi$ on $\partial\mathcal X$
implies the natural boundary condition
$p\,\tfrac{\partial\hat\mu}{\partial n}=0$.
Stationarity for all interior $\varphi$ gives
\[
p\,\rho\,\hat\Lambda(\hat\mu-y)
=2\bar\rho\gamma\,\mathcal L_p\hat\mu,
\]
which is \eqref{eq:euler_mu_general}.

\paragraph{Variation with respect to $\ftp$.}
The relevant terms are 
$\tfrac{1}{2}\hat\Lambda\hat r^2$, 
$-\tfrac12\log\hat\Lambda$, 
and $\bar\gamma\|\nabla\hat\Lambda\|^2$.
We have
\[
\frac{d}{d\varepsilon}\!
\Big(\tfrac{1}{2}\hat\Lambda_\varepsilon\hat r^2
-\tfrac{1}{2}\log\hat\Lambda_\varepsilon\Big)\!\Big|_{0}
=\tfrac{1}{2}\!\left(\hat r^2-\tfrac{1}{\hat\Lambda}\right)\psi,
\]
and
\[
\frac{d}{d\varepsilon}\|\nabla\hat\Lambda_\varepsilon\|^2\Big|_{0}
=2\,\nabla\hat\Lambda\cdot\nabla\psi.
\]
Thus
\[
\delta \mathbf{S}[\hat\Lambda;\psi]
=\int_{\mathcal X}p\,\tfrac{\rho}{2}
\!\left(\hat r^2-\tfrac{1}{\hat\Lambda}\right)\psi\,dx
+2\bar\rho\bar\gamma
\int_{\mathcal X}p\,\nabla\hat\Lambda\cdot\nabla\psi\,dx.
\]

Applying the weighted Green’s identity,
\begin{align*}
\delta \mathbf{S}[\hat\Lambda;\psi]
&=\int_{\mathcal X}
\Big[\,
p\,\tfrac{\rho}{2}\!\left(\hat r^2-\tfrac{1}{\hat\Lambda}\right)
+2\bar\rho\bar\gamma\,\mathcal L_p\hat\Lambda
\Big]\psi\,dx
+2\bar\rho\bar\gamma
\int_{\partial\mathcal X}p\,\psi\,\tfrac{\partial\hat\Lambda}{\partial n}\,dS.
\end{align*}

Stationarity for all $\psi$ on $\partial\mathcal X$
gives $p\,\tfrac{\partial\hat\Lambda}{\partial n}=0$.
Stationarity for all interior $\psi$ yields
\[
p\,\frac{\rho}{2}\!\left[(\hat\mu-y)^2-\tfrac{1}{\hat\Lambda}\right]
=2\bar\rho\bar\gamma\,\mathcal L_p\hat\Lambda,
\]
which is \eqref{eq:euler_lambda_general}.

\paragraph{Conclusion.}
We obtain the coupled PDE system
\begin{align*}
p\rho\,\hat\Lambda(\hat\mu-y)
&=2\bar\rho\gamma\,\mathcal L_p\hat\mu,\\[0.3em]
p\,\tfrac{\rho}{2}\!\left[(\hat\mu-y)^2-\tfrac{1}{\hat\Lambda}\right]
&=2\bar\rho\bar\gamma\,\mathcal L_p\hat\Lambda,
\end{align*}
with homogeneous Neumann data 
$p\,\nabla\hat\mu\cdot n
=p\,\nabla\hat\Lambda\cdot n=0$.
\end{proof}

\begin{remark}[Uniform-density case]
If $p(x)\propto 1$, then $\mathcal{L}_p=\Delta$ and the zero-flux conditions reduce to
$\nabla \ftmu\!\cdot\!\boldsymbol{n}=0$ and $\nabla \ftp\!\cdot\!\boldsymbol{n}=0$.
\end{remark}

\subsubsection{Empirical vs.\ population objective}
\label{rem:empirical_vs_population}
Unless otherwise stated, extremal arguments in this appendix (e.g., the unboundedness results when a regularizer is removed) are stated for the \emph{empirical} FT in which $y(\cdot)$ is treated as a fixed field, equivalently the empirical/Monte Carlo objective induced by a finite dataset. 
In that regime, a sufficiently rich function class can interpolate the observations, i.e., there exist $\hat\mu$ with $\hat r(x):=y(x)-\hat\mu(x)\equiv 0$, and then removing a corresponding regularizer can drive the objective to $-\infty$ via the $-\tfrac12\log\hat\Lambda$ term.

In contrast, for the \emph{population} FT with expectation over $p(y\sep x)$,
\[
\mathbb{E}_{y\sep x}\!\Big[\tfrac12\,\hat\Lambda(x)(y-\hat\mu(x))^2 - \tfrac12\log\hat\Lambda(x)\Big]
= \tfrac12\,\hat\Lambda(x)\Big(\operatorname{Var}[y\sep x] + \big(m(x)-\hat\mu(x)\big)^2\Big) - \tfrac12\log\hat\Lambda(x),
\]
where $m(x):=\mathbb{E}[y\sep x]$. 
If $\operatorname{Var}[y\sep x]>0$ on a set of positive measure, the term growing linearly in $\hat\Lambda$ dominates $-\log\hat\Lambda$ as $\hat\Lambda\to\infty$, and the functional is \emph{not} driven to $-\infty$ by such a blow-up. 
Thus, the unboundedness claims we make in the extreme settings (e.g., \cref{prop:general_extremes}) pertain to the empirical/interpolating regime commonly used in practice, not to the noisily stochastic population risk.

\begin{prop}[Extreme Settings in the General FT]
\label{prop:general_extremes}
Assume $p \in C^1(\overline{\mathcal{X}})$ is strictly positive on a bounded, 
connected domain $\mathcal{X} \subset \mathbb{R}^d$, and that 
$y \in H^1(\mathcal X)$. 
Impose homogeneous Neumann boundary conditions 
$p\nabla \ftmu \!\cdot\! \boldsymbol{n} = 0$ and 
$p\nabla \ftp \!\cdot\! \boldsymbol{n} = 0$ 
on $\partial \mathcal{X}$.
Then for the general field theory
\begin{align}
    \mathbf{S}_{\rho,\gamma}\left[\ftmu,\ftp\right]
    = \int_{\mathcal{X}} p(x)
    \Big[
      -\rho \log \hat{p}(y\sep x)
      + \bar{\rho}\!\left(\gamma\|\nabla \ftmu(x)\|_2^2
      + \bar{\gamma}\|\nabla \ftp(x)\|_2^2\right)
    \Big] dx,
\end{align}
where $\hat{p}(y\mid x)=\mathcal{N}(y\mid \ftmu(x),\ftp(x)^{-1})$, 
the following properties hold:
\begin{enumerate}[label=(\roman*)]
    \item If $\rho=1$, no stationary solution exists.
    \item If $\rho=0$, the solution is non-unique (any constant pair minimizes the functional).
    \item If $\gamma=0$ and $\rho>0$, the objective is unbounded below.
    \item If $\gamma=1$ and $\rho>0$, the objective is likewise unbounded below.
\end{enumerate}
\end{prop}

\begin{proof}
The Euler--Lagrange equations corresponding to $\mathbf{S}_{\rho,\gamma}$, 
using the unnormalized weighted Laplacian 
$\mathcal L_p f := -\nabla\!\cdot(p\nabla f)$, are
\begin{align}
\rho\,p\,\ftp(\ftmu - y)
&= 2\bar{\rho}\,\gamma\,\mathcal{L}_p \ftmu,
\label{eq:EL_mu_extreme}
\\[0.4em]
\frac{\rho}{2}\,p\!\left[(\ftmu - y)^2 - \frac{1}{\ftp}\right]
&= 2\bar{\rho}\,\bar{\gamma}\,\mathcal{L}_p \ftp,
\label{eq:EL_lambda_extreme}
\end{align}
with Neumann boundary conditions
\[
p\nabla \ftmu\!\cdot\!\boldsymbol{n}=0,
\qquad
p\nabla \ftp\!\cdot\!\boldsymbol{n}=0
\quad\text{on }\partial \mathcal{X}.
\]

\paragraph{(i) No regularization $(\rho=1)$.}
Setting $\rho=1$ forces $\bar\rho=0$, so the right-hand sides of 
\eqref{eq:EL_mu_extreme}--\eqref{eq:EL_lambda_extreme} vanish:
\[
p\,\ftp(\ftmu - y)=0,
\qquad
\frac{p}{2}\!\left[\ftp^{-1} - (\ftmu-y)^2\right]=0.
\]
Since $p>0$, this implies
\[
\ftp(\ftmu - y)=0,
\qquad
\ftp^{-1} = (\ftmu - y)^2.
\]
Multiplying the second equation by $\ftp$ gives 
$\ftp(\ftmu - y)^2 = 1$, which contradicts 
$\ftp(\ftmu - y)=0$.  
Thus no stationary point exists.

\paragraph{(ii) No data term $(\rho=0)$.}
Setting $\rho=0$ eliminates the likelihood contribution:
\[
\mathbf{S}_{0,\gamma}[\ftmu,\ftp]
= \int_{\mathcal{X}} p(x)
\big[\gamma\|\nabla \ftmu\|_2^2
+ \bar{\gamma}\|\nabla \ftp\|_2^2\big]dx.
\]
The Euler--Lagrange equations reduce to
\[
\mathcal L_p \ftmu = 0,
\qquad
\mathcal L_p \ftp = 0,
\]
with Neumann boundary conditions.  
On a connected domain with $p>0$, the only solutions are constants, so 
the minimizer is non-unique (any constant pair).

\paragraph{(iii) No mean regularization $(\gamma=0)$.}
With $\gamma=0$ (so $\bar\gamma=1$), the functional becomes
\[
\mathbf{S}_{\rho,0}[\ftmu,\ftp]
= \int_{\mathcal{X}} 
p\,\frac{\rho}{2}\big(\ftp(\ftmu-y)^2 - \log \ftp\big)\,dx
+ \bar\rho\int_{\mathcal{X}} p\,\|\nabla \ftp\|_2^2\,dx.
\]
Choose $\hat\mu \equiv y$ and $\hat\Lambda \equiv C>0$ constant.  
Then $\nabla\hat\Lambda=0$ and $\hat r=y-\hat\mu=0$, giving
\[
\mathbf{S}_{\rho,0}[y,C]
= -\frac{\rho}{2}\Big(\int_{\mathcal X}p(x)\,dx\Big)\log C.
\]
Since $\log C\to\infty$ as $C\to\infty$, the objective tends to $-\infty$.
Thus the functional is unbounded below when $\gamma=0$.

\paragraph{(iv) No variance regularization $(\gamma=1)$ (so $\bar\gamma=0$).}
Now
\[
\mathbf{S}_{\rho,1}[\ftmu,\ftp]
= \int_{\mathcal{X}} 
p\,\frac{\rho}{2}\big(\ftp(\ftmu-y)^2 - \log \ftp\big)\,dx
+ \bar\rho\int_{\mathcal{X}} p\,\|\nabla \ftmu\|_2^2\,dx.
\]
Again take $\hat\mu \equiv y$ and $\hat\Lambda \equiv C>0$.  
Then $\nabla\hat\mu=\nabla y \in L^2$ and $\nabla\hat\Lambda=0$, so
\[
\mathbf{S}_{\rho,1}[y,C]
= -\frac{\rho}{2}\Big(\int_{\mathcal X} p(x)\,dx\Big)\log C
+ \bar\rho\int_{\mathcal{X}} p\,\|\nabla y\|_2^2\,dx.
\]
The second term is finite and independent of $C$, while the first 
tends to $-\infty$ as $C\to\infty$.  
Thus the objective is unbounded below when $\gamma=1$.
\end{proof}

\begin{cor}[Necessity of two-sided regularization for $\rho>0$ (general $p$)]
\label{cor:two_sided_needed_general}
Under the assumptions of \cref{prop:general_extremes} with $\rho>0$, any well-posed formulation requires
\[
    \gamma\in(0,1)\quad\Longleftrightarrow\quad 
    \alpha=\tfrac{\bar\rho}{\rho}\gamma>0
    \ \ \text{and}\ \ 
    \beta=\tfrac{\bar\rho}{\rho}\bar\gamma>0.
\]
Equivalently, if either $\alpha=0$ or $\beta=0$, the objective is unbounded below.
\end{cor}

\begin{proof}
If $\alpha=0$ (i.e., $\gamma=0$), part (iii) of \cref{prop:general_extremes} shows that the functional $\mathbf{S}_{\rho,0}$ is unbounded below: choosing $\hat\mu \equiv y$ and $\hat\Lambda \equiv C$ with $C\to\infty$ nulls the residual term while the $-\log \hat\Lambda$ contribution drives $\mathbf{S}_{\rho,0}[y,C]\to -\infty$. 
Similarly, if $\beta=0$ (i.e., $\bar\gamma=0$), part (iv) shows that the same construction yields $\mathbf{S}_{\rho,1}[y,C]\to -\infty$. 
Thus in either case ($\alpha=0$ or $\beta=0$) the objective is unbounded below, proving the necessity of $\alpha,\beta>0$ for well-posedness when $\rho>0$.
\end{proof}

\begin{prop}[Existence of a minimizer for interior regularization weights]
\label{prop:interior_exist}
Let $\mathcal X \subset \mathbb R^d$ be a bounded, connected Lipschitz domain,
and let $p \in C^1(\overline{\mathcal X})$ be strictly positive on
$\overline{\mathcal X}$. Fix $\rho,\gamma \in (0,1)$ with
$\bar\rho = 1-\rho$ and $\bar\gamma = 1-\gamma$. Assume that the observed data field satisfies $y \in H^1(\mathcal X)$ (and hence
$y \in L^2(\mathcal X)$).

Fix constants $0 < \lambda_{\min} < \lambda_{\max} < \infty$, and define the
admissible set
\[
\mathcal{A}
:= \Big\{ (\ftmu, \ftp) \in H^1(\mathcal{X}) \times H^1(\mathcal{X}) \,:\,
\lambda_{\min} \leq \ftp(x) \leq \lambda_{\max} \ \text{a.e.\ in }\mathcal X
\Big\}.
\]
Then the variational objective
\[
\mathbf{S}_{\rho,\gamma}[\ftmu,\ftp]
= \int_{\mathcal{X}} p(x)
\Big[-\rho \log \hat{p}(y\sep x)
+ \bar{\rho}\big(\gamma\|\nabla \ftmu(x)\|_2^2
+ \bar{\gamma}\|\nabla \ftp(x)\|_2^2 \big)\Big] dx,
\]
where $\hat{p}(y \sep x) = \mathcal{N}(y \sep \ftmu(x), \ftp(x)^{-1})$, admits
at least one minimizer $(\ftmu^*, \ftp^*) \in \mathcal{A}$.
\end{prop}

\begin{proof}
We apply the direct method of the calculus of variations
(see~\citet[Sec.~8.2]{evans_partial_2010}). Because $p$ is continuous and strictly
positive on the bounded domain $\mathcal X$, it is bounded above and below by
positive constants, so the weighted and unweighted $L^2$ and $H^1$ norms are
equivalent.

\paragraph{1. Coercivity.}
Writing $(\hat\mu,\hat\Lambda)$ for a generic admissible pair, the functional
can be expressed as
\[
\mathbf{S}_{\rho,\gamma}[\hat\mu, \hat\Lambda]
= \int_{\mathcal{X}} p(x) \Big[
\frac{\rho}{2} \big( \hat\Lambda(x)(\hat\mu(x) - y(x))^2
- \log \hat\Lambda(x) \big)
+ \bar\rho \big( \gamma \|\nabla \hat\mu(x)\|_2^2
+ \bar\gamma \|\nabla \hat\Lambda(x)\|_2^2 \big)
\Big] dx.
\]
Since $\hat\Lambda(x) \in [\lambda_{\min},\lambda_{\max}]$ a.e., we have, for
all $x\in\mathcal X$,
\[
\hat\Lambda(x)(\hat\mu(x)-y(x))^2 - \log \hat\Lambda(x)
\;\ge\;
\lambda_{\min}(\hat\mu(x)-y(x))^2 - \log \lambda_{\max}.
\]
Multiplying by $p(x)\frac{\rho}{2}$ and integrating over $\mathcal X$ yields
\begin{align*}
\int_{\mathcal X} p(x)\,\frac{\rho}{2}
\big( \hat\Lambda(x)&(\hat\mu(x)-y(x))^2 - \log \hat\Lambda(x) \big)\,dx\\
&\;\ge\;
\int_{\mathcal X} p(x)\,\frac{\rho}{2}
\big( \lambda_{\min}(\hat\mu(x)-y(x))^2 - \log \lambda_{\max} \big)\,dx \\[0.4em]
&=
\frac{\rho\lambda_{\min}}{2}
\int_{\mathcal X} p(x)\,(\hat\mu(x)-y(x))^2\,dx
\;-\;
\frac{\rho\log\lambda_{\max}}{2}\int_{\mathcal X} p(x)\,dx.
\end{align*}
Since $p$ is continuous and strictly positive on the bounded domain
$\overline{\mathcal X}$, there exist constants
$0 < p_{\min} \le p(x) \le p_{\max} < \infty$ for all $x\in\overline{\mathcal X}$.
Using $p(x)\ge p_{\min}$, we obtain
\[
\int_{\mathcal X} p(x)\,(\hat\mu(x)-y(x))^2\,dx
\;\ge\;
p_{\min} \int_{\mathcal X} (\hat\mu(x)-y(x))^2\,dx
= p_{\min}\,\|\hat\mu - y\|_{L^2(\mathcal X)}^2.
\]
Thus the data term is bounded below by
\[
\int_{\mathcal X} p(x)\,\frac{\rho}{2}
\big( \hat\Lambda(\hat\mu-y)^2 - \log \hat\Lambda \big)\,dx
\;\ge\;
c_1 \|\hat\mu - y\|_{L^2(\mathcal X)}^2 - C_0,
\]
for some constants $c_1>0$ and $C_0>0$ depending on
$(\rho,\lambda_{\min},\lambda_{\max},p)$.

For the gradient terms, we similarly have
\begin{align*}
\int_{\mathcal X} p(x)\,\gamma \|\nabla\hat\mu(x)\|_2^2\,dx
&\;\ge\;
\gamma p_{\min} \int_{\mathcal X} \|\nabla\hat\mu(x)\|_2^2\,dx
= \gamma p_{\min}\,\|\nabla\hat\mu\|_{L^2(\mathcal X)}^2  \quad \text{and}\\
\int_{\mathcal X} p(x)\,\bar\gamma \|\nabla\hat\Lambda(x)\|_2^2\,dx
&\;\ge\;
\bar\gamma p_{\min} \int_{\mathcal X} \|\nabla\hat\Lambda(x)\|_2^2\,dx
= \bar\gamma p_{\min}\,\|\nabla\hat\Lambda\|_{L^2(\mathcal X)}^2.
\end{align*}
Combining these estimates, we obtain
\[
\mathbf{S}_{\rho,\gamma}[\hat\mu,\hat\Lambda]
\;\ge\;
c_1 \|\hat\mu - y\|_{L^2(\mathcal X)}^2
+ c_2 \Big( \|\nabla \hat\mu\|_{L^2(\mathcal X)}^2
+ \|\nabla \hat\Lambda\|_{L^2(\mathcal X)}^2 \Big)
- C,
\]
for some constants $c_1,c_2,C>0$ depending on
$(\rho,\gamma,\lambda_{\min},\lambda_{\max},p)$ and $y$.

Since
\[
\|\hat\mu\|_{L^2(\mathcal X)} \le \|\hat\mu-y\|_{L^2(\mathcal X)} + \|y\|_{L^2(\mathcal X)}
\]
and
\[
\|\hat\Lambda\|_{L^2(\mathcal X)}^2
= \int_{\mathcal X} |\hat\Lambda(x)|^2\,dx
\le \lambda_{\max}^2 |\mathcal X|
\quad\text{for all admissible $\hat\Lambda$},
\]
we can absorb these constants to obtain
\[
\mathbf{S}_{\rho,\gamma}[\hat\mu,\hat\Lambda]
\;\ge\;
C_1\Big( \|\hat\mu\|_{H^1(\mathcal X)}^2
+ \|\hat\Lambda\|_{H^1(\mathcal X)}^2 \Big) - C_2,
\]
for some $C_1,C_2>0$, where
\[
\|\hat\mu\|_{H^1(\mathcal X)}^2
:= \|\hat\mu\|_{L^2(\mathcal X)}^2 + \|\nabla\hat\mu\|_{L^2(\mathcal X)}^2,
\quad
\|\hat\Lambda\|_{H^1(\mathcal X)}^2
:= \|\hat\Lambda\|_{L^2(\mathcal X)}^2 + \|\nabla\hat\Lambda\|_{L^2(\mathcal X)}^2.
\]
Thus $\mathbf S_{\rho,\gamma}$ is coercive on $\mathcal A$.

\paragraph{2. Minimizing sequence and compactness.}
Let $(\hat\mu_n,\hat\Lambda_n)\in\mathcal A$ be a minimizing sequence, i.e.
\[
    \lim_{n\to\infty} \mathbf S_{\rho,\gamma}[\hat\mu_n,\hat\Lambda_n]
    = \inf_{(\hat\mu,\hat\Lambda)\in\mathcal A}
    \mathbf S_{\rho,\gamma}[\hat\mu,\hat\Lambda].
\]
By coercivity (Step~1), the sequence $(\hat\mu_n,\hat\Lambda_n)$ is bounded in $H^1(\mathcal X)\times H^1(\mathcal X)$.  
Since $H^1(\mathcal X)$ is a Hilbert (hence reflexive) space, every bounded sequence has a weakly convergent subsequence.  
Passing to such a subsequence (not relabeled), there exists $(\hat\mu^*,\hat\Lambda^*)\in H^1(\mathcal X)^2$ such that
\[
    \hat\mu_n \rightharpoonup \hat\mu^*
    \quad\text{and}\quad
    \hat\Lambda_n \rightharpoonup \hat\Lambda^*
    \quad\text{weakly in } H^1(\mathcal X).
\]

Weak convergence controls functions and their gradients only in an averaged sense, which is insufficient to pass to the limit in the nonlinear term $\hat\Lambda(\hat\mu-y)^2$.  
To obtain pointwise and $L^2$ convergence, we use a standard compactness result: on any bounded Lipschitz domain, the Sobolev embedding $H^1(\mathcal X)\hookrightarrow L^2(\mathcal X)$ is \emph{compact} (Rellich--Kondrachov).  
Thus, up to a further subsequence,
\[
    \hat\mu_n \to \hat\mu^*,\qquad
    \hat\Lambda_n \to \hat\Lambda^*
    \quad\text{strongly in } L^2(\mathcal X),
\]
and in particular almost everywhere on~$\mathcal X$.

Because each $(\hat\mu_n,\hat\Lambda_n)$ lies in the admissible set $\mathcal A$, we have the pointwise bounds $\lambda_{\min}\le \hat\Lambda_n(x)\le\lambda_{\max}$ a.e.\
Strong $L^2$ convergence implies almost-everywhere convergence along the subsequence, so these bounds pass to the limit:
\[
    \lambda_{\min} \le \hat\Lambda^*(x) \le \lambda_{\max}
    \quad\text{for a.e.\ }x.
\]
Hence $(\hat\mu^*,\hat\Lambda^*)\in\mathcal A$.

\paragraph{3. Lower semicontinuity and passing to the limit.}
The gradient regularizers
\[
    \int_{\mathcal X} p(x)\,\gamma\|\nabla \hat\mu(x)\|_2^2\,dx,
    \qquad
    \int_{\mathcal X} p(x)\,\bar\gamma\|\nabla \hat\Lambda(x)\|_2^2\,dx
\]
are weakly lower semicontinuous in $H^1(\mathcal X)$, as they are convex quadratic forms in the gradients.  

For the data-fitting term, the strong $L^2$ convergence of $\hat\mu_n$ and $\hat\Lambda_n$, together with the uniform bounds $\lambda_{\min}\le \hat\Lambda_n\le\lambda_{\max}$ and continuity of the map
\[
    (u,\lambda)\mapsto \lambda(u-y)^2 - \log \lambda
    \quad
    \text{on }\mathbb R\times[\lambda_{\min},\lambda_{\max}],
\]
implies
\begin{align}
\int_{\mathcal X} p(x)\,
\Big[\hat\Lambda_n(x)(\hat\mu_n(x)-y(x))^2 - &\log \hat\Lambda_n(x)\Big] dx \\ 
&\longrightarrow\;
\int_{\mathcal X} p(x)\,
\Big[\hat\Lambda^*(x)(\hat\mu^*(x)-y(x))^2 - \log \hat\Lambda^*(x)\Big] dx.
\end{align}
Thus the data term is continuous along the minimizing subsequence.

Combining the weak lower semicontinuity of the gradient terms with the continuity of the data term yields
\[
    \mathbf S_{\rho,\gamma}[\hat\mu^*,\hat\Lambda^*]
    \;\le\;
    \liminf_{n\to\infty} \mathbf S_{\rho,\gamma}[\hat\mu_n,\hat\Lambda_n]
    =
    \inf_{(\hat\mu,\hat\Lambda)\in\mathcal A}
    \mathbf S_{\rho,\gamma}[\hat\mu,\hat\Lambda].
\]
Thus $(\hat\mu^*,\hat\Lambda^*)$ attains the infimum of $\mathbf S_{\rho,\gamma}$ over $\mathcal A$.  
This completes the existence proof via the direct method of the calculus of variations; see~\citet[Sec.~8.2, Thm.~2]{evans_partial_2010}.
\end{proof}

\begin{remark}[Interpretation of Bounded Precision Assumption]\label{rem:bounded_prec}
The restriction \( \lambda_{\min} \leq \hat\Lambda(x) \leq \lambda_{\max} \) ensures that the predicted precision (inverse variance) remains within a physically and statistically meaningful range. 
From a modeling standpoint, this prevents pathological behavior:
\begin{itemize}
    \item Allowing \( \hat\Lambda(x) \to 0 \) corresponds to arbitrarily large predictive variance, i.e., extreme uncertainty, which is typically uninformative and numerically unstable.
    \item Allowing \( \hat\Lambda(x) \to \infty \) implies vanishing predictive variance, i.e., extreme overconfidence, even in regions with limited or noisy data—this can lead to poor generalization.
\end{itemize}
It is important to note that this assumption is made purely for mathematical tractability: the existence of such bounds is sufficient for the argument, and they need not be tight. 
For instance, one may take \( \lambda_{\min} = 2^{-100} \), \( \lambda_{\max} = 2^{100} \) and the proof still holds. 
\end{remark}

\section{Bayesian Field Theory (Supplementary Derivations)}
\label{app:bft_appendix}

This section provides the explicit mathematical formulation and discretization details 
for the Bayesian reformulation of the field theory (BFT) introduced in 
\cref{sec:bft_main}.  
It formalizes the two parameterizations of the variance field, derives the MAP and 
functional gradients, and outlines the weak convergence of the corresponding 
discretized priors.

\subsection{MAP-equivalent parameterization}
\label{app:bft_appendix_additive}

Let $\bar\rho := 1 - \rho$ and $\bar\gamma := 1 - \gamma$.  
We parameterize $\ftp=e^{\hat{\eta}}$ to enforce $\ftp>0$ and choose priors
so that the deterministic FT regularizers are recovered exactly at the posterior mode,
making the MAP equations identical to those of the FT.

\subsubsection{Priors}
\begin{align}
-\log \pi(\hat\mu)
&= \tfrac{\gamma}{2}\int_{\mathcal X} p(x)\,\|\nabla\hat\mu(x)\|^2\,dx,\\
-\log \pi(\hat{\eta})
&= \tfrac{\bar\gamma}{2}\int_{\mathcal X} p(x)\,e^{2\hat{\eta}(x)}\,\|\nabla\hat{\eta}(x)\|^2\,dx.
\end{align}
Homogeneous Neumann boundary conditions
$p\nabla\hat\mu\!\cdot\!\mathbf{n}
 =p\,e^{2\hat{\eta}}\nabla\hat{\eta}\!\cdot\!\mathbf{n}=0$
ensure vanishing boundary terms.  
These priors act as Gaussian Markov random field (GMRF) and log-convex field priors, respectively, enforcing smoothness and positivity.

\subsubsection{Scaling conventions between FT and BFT}
\label{app:ft_bft_scaling}
The field--theoretic functional in \cref{eq:ft_def} omits the global 
$\tfrac{1}{2}$ factor in its Dirichlet energies for notational simplicity 
and consistency with our simulations.
In the Bayesian Field Theory (BFT) formulation, an equivalent $\tfrac{1}{2}$ 
appears in the Gaussian prior and likelihood log--densities 
(see~\cref{eq:bft_priors_main}).
This factor can be absorbed into the definition of the prior precision or covariance, 
so it merely rescales the regularization weights without changing the stationary 
equations.  
Consequently, the posterior mode and the field--theoretic stationary conditions remain identical, and the equivalence $\mathrm{MAP}\equiv\mathrm{FT}$ holds exactly.

\subsubsection{Posterior and functional gradients}
Combining likelihood and priors gives
\begin{align}
\Phi(\hat\mu,\hat\eta)
=\int p\,\Big\{
 \rho\big[\tfrac12 e^{\hat\eta}(y-\hat\mu)^2-\tfrac12\hat\eta\big]
 +\tfrac{\bar\rho}{2}\big[\gamma\|\nabla\hat\mu\|^2
 +\bar\gamma\,e^{2\hat\eta}\|\nabla\hat\eta\|^2\big]
\Big\}.
\end{align}

Its gradients (before discretization) are:
\begin{align}
\nabla_{\hat\mu}\Phi
&=\rho\,p\,e^{\hat{\eta}}(\hat\mu-y)
 -\bar\rho\,\gamma\,\nabla\!\cdot(p\nabla\hat\mu),\\
\nabla_{\hat{\eta}}\Phi
&=\tfrac{\rho\,p}{2}\!\left(e^{\hat{\eta}}(y-\hat\mu)^2-1\right)
 -\bar\rho\,\bar\gamma\,\nabla\!\cdot(p\,e^{2\hat{\eta}}\nabla\hat{\eta})
 +\bar\rho\,\bar\gamma\,p\,e^{2\hat{\eta}}\|\nabla\hat{\eta}\|^2.
\end{align}
At $\nabla\Phi=0$, these coincide with the FT Euler–Lagrange equations.

\paragraph{Note on normalization}
The prior on $\hat{\eta}$ guarantees $\hat\Lambda>0$ but leaves constant (mean) modes unpenalized.  
To ensure a proper posterior, one may add a small stabilizer term
$\frac{\varepsilon}{2}\int p(\hat{\eta}-\hat{\eta}_0)^2\,dx$, 
which normalizes the prior without affecting the stationary equations or MAP solution.

\subsection{Discretization and weak convergence}
\label{app:bft_discrete_appendix}

Let $\mathcal{G}_h$ denote a shape--regular mesh of $\mathcal{X}$ with mesh size $h = \max_{K\in\mathcal{G}_h} \mathrm{diam}(K) \to 0$.
Shape--regularity means that all elements (triangles, tetrahedra, or grid cells) have uniformly bounded aspect ratio, i.e., no element becomes arbitrarily thin or degenerate as the mesh is refined.
This ensures numerical stability of the discrete gradient and Laplacian operators.

The mesh need not be uniform.
In our setting, the node density of $\mathcal{G}_h$ is proportional to the sampling density $p(x)$, corresponding to the empirical discretization used throughout the field theory.
Regions of higher $p(x)$ therefore receive finer resolution, while maintaining shape--regularity in each local neighborhood.

Denote by $\mathbf{L}_{p,h}$ the standard finite--element (or weighted graph) discretization of the weighted Laplacian
\[
\mathcal{L}_p f = -\nabla\!\cdot(p\nabla f)
\]
with homogeneous Neumann boundary conditions (zero normal flux on boundary facets).
Let $\nabla_h$ be the discrete gradient operator and
$\mathbf{M}_{p,h}$ the (possibly lumped) $p$--weighted mass matrix.
The discrete priors then read
\[
-\log \pi(\hat{\boldsymbol{\mu}}) 
= \tfrac{\gamma}{2}\,\hat{\boldsymbol{\mu}}^\top \mathsf{L}_{p,h}\, \hat{\boldsymbol{\mu}},
\qquad
-\log \pi(\hat{\boldsymbol{\eta}}) 
= \tfrac{\bar{\gamma}}{2}\, 
  \big(\mathrm{e}^{\hat{\boldsymbol{\eta}}}\!\odot\nabla_h \hat{\boldsymbol{\eta}}\big)^\top
  \mathbf{M}_{p,h}\,
  \big(\mathrm{e}^{\hat{\boldsymbol{\eta}}}\!\odot\nabla_h \hat{\boldsymbol{\eta}}\big),
\]
where $\odot$ denotes elementwise multiplication.

Under standard finite--element assumptions
(shape--regularity, bounded domain, and $h\!\to\!0$),
these discrete priors converge weakly to their continuous counterparts.
If the node density of $\mathcal{G}_h$ tracks $p(x)$, then $\mathbf{L}_{p,h}$ and $\mathbf{M}_{p,h}$ provide consistent approximations of the weighted operators associated with $\mathcal{L}_p$.
This includes the graph--based discretizations used in practice, which can be viewed as stochastic FEM quadratures under the empirical measure $p(x)\,dx$.
For rigorous convergence results in this setting, see \citet{lindgren_explicit_2011}, who prove weak convergence of discrete GMRFs and SPDE operators to their continuous Mat\'{e}rn and elliptic limits.

\subsection{Gaussian Field–Process Equivalence}
\label{app:bft_gp_equivalence}

A Gaussian random field (GRF) $f(x)$ is a Gaussian measure over functions $f:\mathcal{X}\!\to\!\mathbb{R}$ specified by a precision (inverse-covariance) operator~$\mathcal{L}$.  
Formally, the prior density (up to normalization) is
\begin{equation}
    p(f)\;\propto\;
    \exp\!\left[-\tfrac{1}{2}\!\int_{\mathcal{X}}
    f(x)\,(\mathcal{L}f)(x)\,dx\right],
    \label{eq:grf_prior}
\end{equation}
where $\mathcal{L}$ is typically a positive-definite elliptic operator such as the (weighted) Laplacian
\(
\mathcal{L}_p f = -\nabla\!\cdot(p(x)\nabla f(x)).
\)

The Green’s function $G(x,x')$ of~$\mathcal{L}$ satisfies
\(
\mathcal{L}G(x,\!\cdot\,)=\delta(x-\!\cdot\,),
\)
and defines the covariance kernel
\begin{equation}
    k(x,x') \;=\; G(x,x')
    \;=\; \big(\mathcal{L}^{-1}\big)(x,x').
    \label{eq:green_gp}
\end{equation}
Hence a GRF with precision~$\mathcal{L}$ is equivalent to a Gaussian process $\mathcal{GP}(0,k)$ whose kernel is the inverse of~$\mathcal{L}$:
\[
f \sim \mathcal{N}(0, \mathcal{L}^{-1})
\quad\Longleftrightarrow\quad
f(x)\sim \mathcal{GP}(0,k(x,x')).
\]

For example, the Dirichlet-energy prior used throughout the field theory,
\begin{equation}
    p(\hat\mu)
    \propto
    \exp\!\Big[-\tfrac{\gamma}{2}\!\int_{\mathcal X}p(x)\,
    \|\nabla\hat\mu(x)\|^2dx\Big],
    \label{eq:dirichlet_gp_prior}
\end{equation}
is a zero-mean GRF with precision operator $\gamma\,\mathcal{L}_p$ and thus corresponds to a GP with covariance operator $(\gamma\,\mathcal{L}_p)^{-1}$.
Consequently, the Dirichlet energy acts as the negative log-density of a Gaussian process whose kernel is the Green’s function of the weighted Laplacian~$\mathcal{L}_p$:
\[
k_p(x,x') = \gamma^{-1} G_p(x,x'),
\qquad
\mathcal{L}_p G_p(x,\!\cdot\,)=\delta(x-\!\cdot\,).
\]

This operator-based Gaussian prior is consistent with the classical RKHS–Bayesian correspondence of \citet{kimeldorf_wahba_bayes_1970}, in which quadratic smoothness penalties  correspond to Gaussian priors whose covariance operators are the inverses of the associated differential operators.

This is a concrete instance of the general operator–kernel correspondence used to construct Gaussian fields from elliptic SPDEs, including the Mat\'ern  family \citep{lindgren_explicit_2011}, and it links our field-theoretic prior  directly to GP and RKHS smoothness priors.

\section{Experimental Details}
\label{app:exp_details}
\subsection{Datasets}
\label{app:datasets}

\begin{figure}
\centering
\includegraphics[width=.6\textwidth]{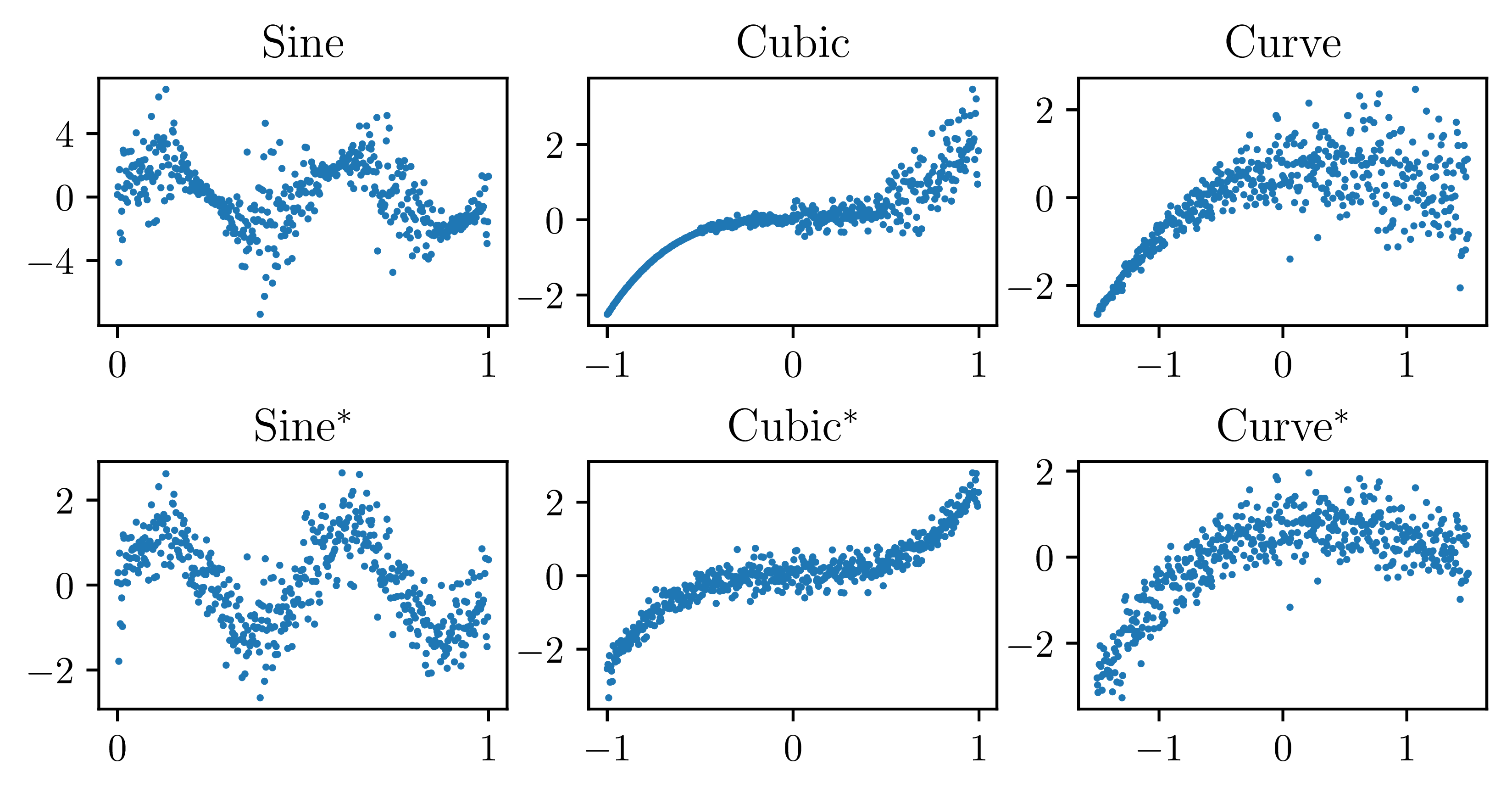}
\caption{Visualization of heteroskedastic and homoskedastic versions of simulated datasets. Specific details for the functional form of these can be found in \cref{tab:sim_data}.}
\label{fig:sim_data_vis}
\end{figure}

We chose 64 datapoints in each of the simulated datasets. The generating processes for each simulated dataset is included in \cref{tab:sim_data} and can be seen in \cref{fig:sim_data_vis}. The homoskedastic data is simulated in the same way, but with $f(x) = 1$. For testing, we simulate a new dataset of 64 datapoints with the same process. \cref{tab:uci_data} summarizes the UCI datasets. We provide a description of the ClimSim climate data in Appendix~\ref{app:climsimdesc}.
\begin{table*} %
\centering
\footnotesize
\caption{Simulated datasets. Each dataset is defined by a true $\mu$ function and then a noise function $f$. All data is generated as $\mu(x) + \epsilon(x)$ where $\epsilon(x) \sim \mathcal{N}(0, f(x)^2)$. After the datasets were generated they were scaled to have mean zero and standard deviation one. The homoskedastic versions of each dataset fix $f(x)=1$. The datasets are shown in \cref{fig:sim_data_vis}.}
\begin{tabular}{ llll }
\toprule 
Dataset  & Mean ($\mu$) & Noise Pattern ($f$) & Domain
\\
\midrule
Sine &  $\mu(x) = 2 \sin(4\pi x)$ & $f(x) = \sin(6\pi x) + 1.25$ & $x\in [0, 1]$\\
\midrule 
Cubic   &  $\mu(x) = x^3 
\; $ & $f(x) = \begin{cases}0.1 & \text{for } x < -0.5\\1 & \text{for }x \in [-0.5, 0.0) \\ 3 & \text{for } x \in [0.0, 0.5) \\ 10 & \text{for } x \ge .5 \end{cases}$&$x\in [-1, 1]$
\\
\midrule
Curve  &  $\mu(x) = x - 2x^2 + 0.5x^3$ & $f(x) = x + 1.5$ & $x \in [-1.5, 1.5]$
\\
\bottomrule
\end{tabular}
\label{tab:sim_data}
\end{table*}

\begin{table*} %
\centering
\footnotesize
\caption{UCI dataset details.}

\begin{tabular}{ lccc }
\toprule 
Dataset & Train Size & Test Size & Input Dimension 
\\
\midrule
Concrete &  687 & 343 & 8 \\
\midrule 
Housing   &  337 & 168 & 13 
\\
\midrule
Power  &  6379 & 3189 & 4 
\\
\midrule
Yacht & 204 & 102 & 6
\\
\bottomrule
\end{tabular}
\label{tab:uci_data}
\end{table*}

\subsection{Training Details}
\label{app:training}

For the neural-network experiments, we evaluated 22 values of $(\rho,\gamma)$ spaced log-uniformly on the logit scale between $10^{-10}$ and $1-10^{-5}$.  
For the field-theoretic models we used 20 values between $10^{-6}$ and $1-10^{-7}$, again logit-spaced.  
The ranges differ slightly due to numerical stability during fitting.
Along the line $\rho = 1-\gamma$, we sampled 100 values between $10^{-11}$ and $1$, using a dense log-uniform grid near the extremes and a uniform grid on $[0.1,0.9]$.
The limiting cases $\rho,\gamma \in \{0,1\}$ were excluded for stability.
The exact grids used for all sweeps are provided in the public repository.
All experiments were run on Nvidia Quadro RTX~8000 GPUs, totaling approximately 
500 GPU hours.

\subsection{Field Theory}
\label{app:FT}
For the discretized field theory we take $n_{ft}=4096$ evenly spaced points on the interval $[-1, 1]$. 
There are two datapoints placed beyond $[-1, 1]$ because the method we use to estimate the gradients requires the datapoints to have left and right neighbors. These datapoints were not included when computing our metrics. 
Of these 4096 datapoints 64 were randomly selected to be used for training neural networks $\nnmu, \nnLambda$.
The field theory results were consistent across choices of $n_{ft}\in\{256, 512, 1024, 2048, 4096\}$. 
We present results for $n_{ft}=4096$ in the main paper. 
We train for $100000$ epochs and use the Adam optimizer with a basic triangular cycle that scales initial amplitude by half each cycle on the learning rate. The minimum and maximum learning rates were 0.0005 and 0.01. The cycles were 5000 epochs long. We clip the gradients at 1000.

\subsection{Neural Network Training for Synthetic and UCI Data}
\label{app:nn-training}

Across all experiments we train the networks using Adam with a triangular cyclic learning rate schedule (min.\ 0.0001, max.\ 0.01), where each cycle reduces its amplitude by half.  
Unless otherwise noted, the cycle length is 50{,}000 epochs and gradients are clipped at 1000.
Training proceeds in two phases: we first train only the mean network $\nnmu$ for the initial portion of training, and then train both $\nnmu$ and $\nnLambda$ for the remainder.

\paragraph{Synthetic datasets.}
For all synthetic datasets except \emph{Sine}, we train for 600{,}000 epochs: the first 250{,}000 epochs train only $\nnmu$, followed by 350{,}000 epochs training both networks.  
For the \emph{Sine} dataset, we use the same setup but train for 2{,}500{,}000 epochs total.

\paragraph{UCI datasets.}
For \emph{Concrete}, \emph{Housing}, and \emph{Yacht}, we train for 500{,}000 epochs: 250{,}000 epochs on $\nnmu$ alone and 250{,}000 epochs jointly on $\nnmu$ and $\nnLambda$, using the same learning rate schedule.
For the \emph{Power} dataset, minibatching is required due to its size.  
We use a batch size of 1000 and train for 50{,}000 epochs in total, with 25{,}000 epochs devoted to $\nnmu$ alone and the remaining 25{,}000 epochs training both
networks.  
The same cyclic learning-rate schedule is used, but with a reduced cycle length of 5000.

\section{Additional Results}
\label{app:add_res}
Both FT and neural networks were fit to the heteroskedastic and homoskedastic synthetic datasets described in \cref{tab:sim_data}. The main results for these displayed as phase diagrams of various metrics can be seen in \cref{fig:synth_ft} and \cref{fig:synth_nn} respectively. We largely see the same trends as were exhibited by the real-world datasets seen in \cref{fig:summary_mean}.

\begin{figure}
\centering
\begin{subfigure}[t]{0.48\textwidth}
    \centering
    \includegraphics[width=\textwidth]{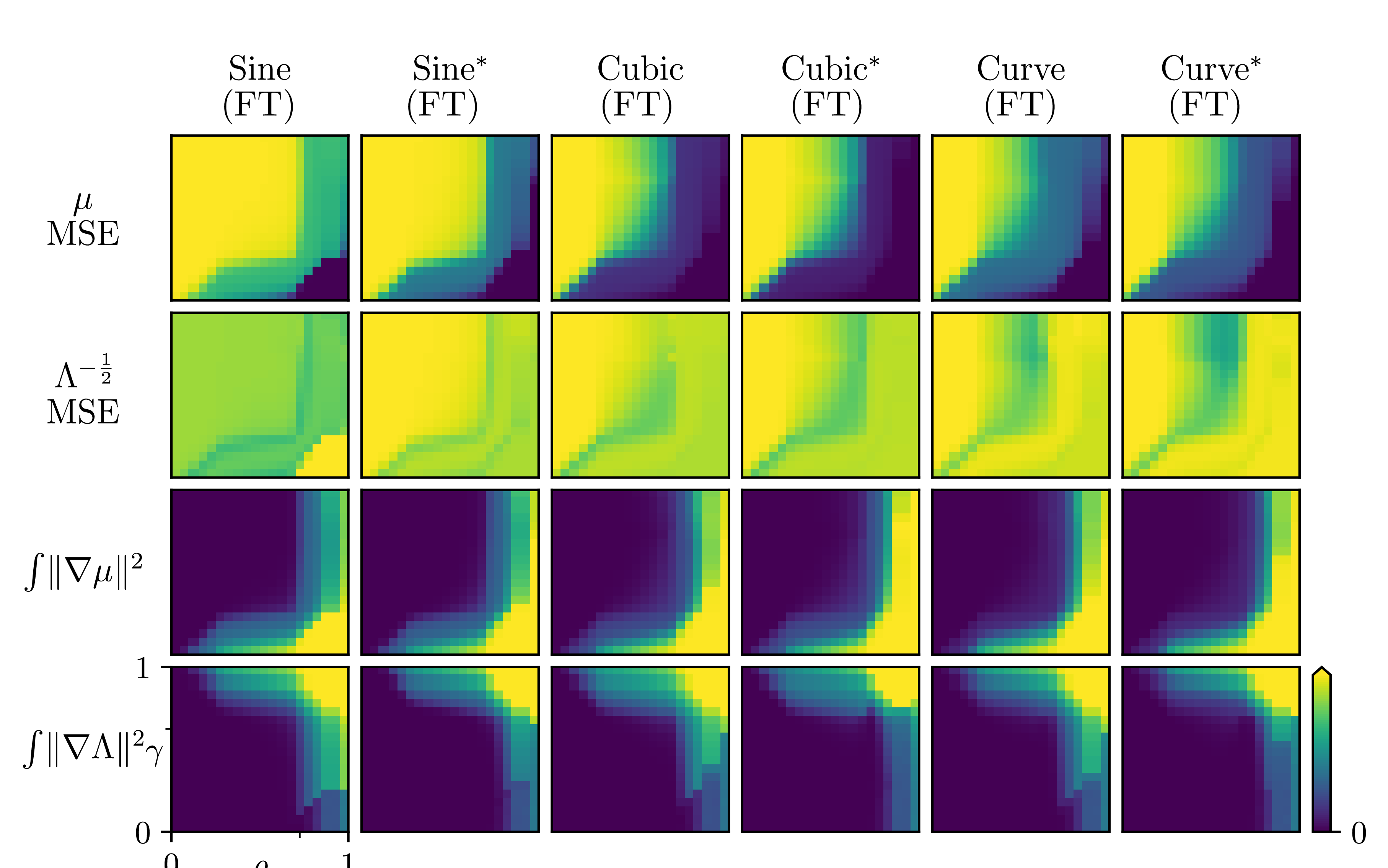}
    \caption{Field theory.}
    \label{fig:synth_ft}
\end{subfigure}
\hfill
\begin{subfigure}[t]{0.48\textwidth}
    \centering
    \includegraphics[width=\textwidth]{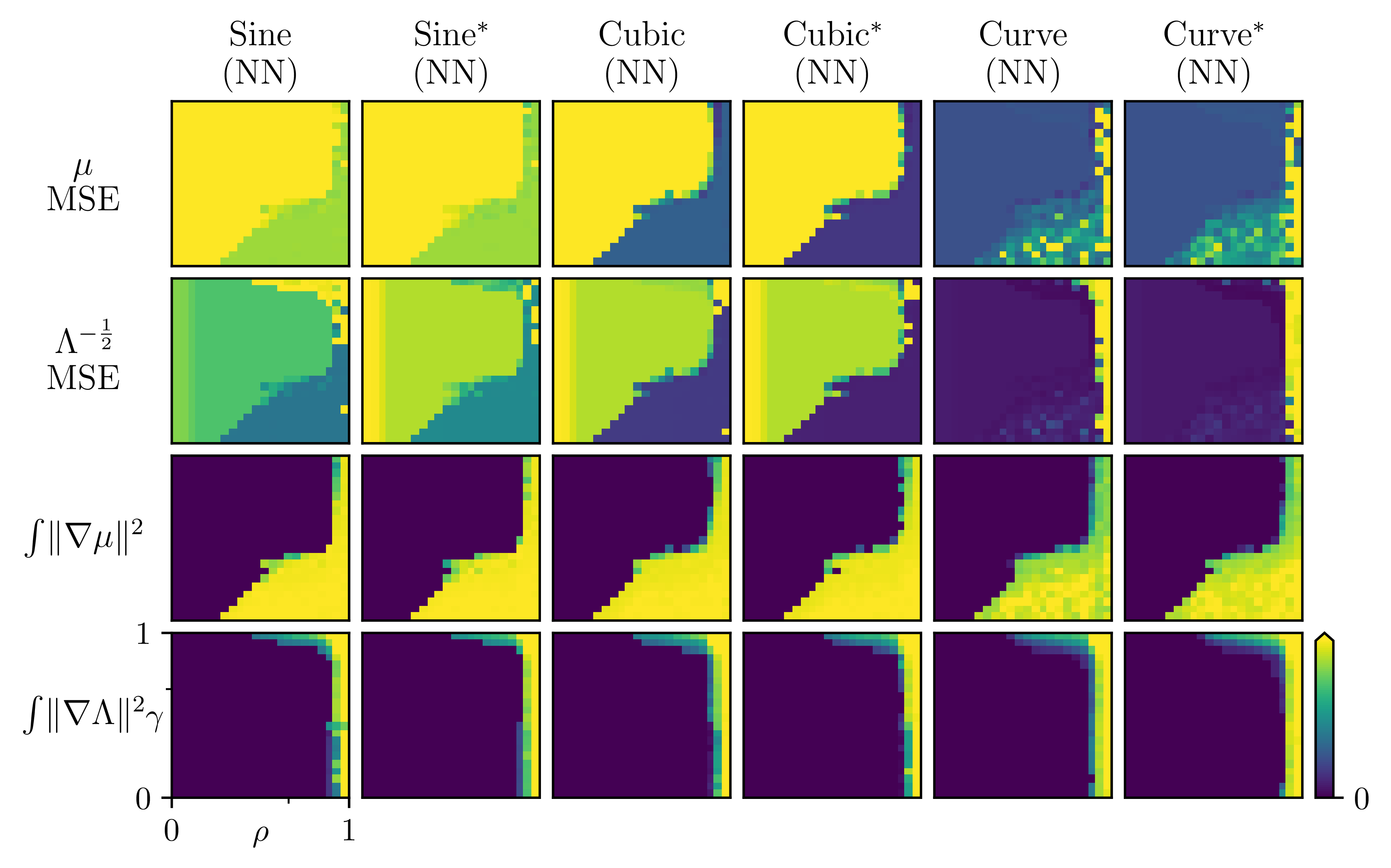}
    \caption{Neural network.}
    \label{fig:synth_nn}
\end{subfigure}

\caption{Phase diagrams for the field theory (\emph{left}) and neural networks (\emph{right}) on six synthetic datasets, in the same configuration as \cref{fig:summary_mean}. Dataset names with an $*$ denote homoskedastic counterparts.}
\label{fig:synth_ft_nn}
\end{figure}

\section{Comparison to Baselines}\label{app:baseline_comp}
We compare the performance of our diagonal $\rho=1-\gamma$ search against two baselines, $\beta$-NLL \citep{seitzer_pitfalls_2022} and an ensemble of six MLE-fit heteroskedastic regression models \citep{lakshminarayanan_simple_2017}. We use $\mu$ MSE, $\Lambda^{-\frac{1}{2}}$ MSE, and expected calibration error (ECE) to evaluate the models. In all cases lower values are better. Note that the method of ensembling multiple individual MLE-fit models from \cite{lakshminarayanan_simple_2017} could be implemented on our method or $\beta$-NLL as well.

\subsection{Model Architecture}
All (individual) models have the same architecture: fully connected neural networks with three hidden layers of 128 nodes and leaky ReLU activations for the synthetic and UCI datasets and fully connected neural networks with three hidden layers of 256 nodes for the ClimSim data \citep{yu_climsim_2023}. 
Note that both baselines model the variance while our approach models the precision (inverse-variance).
In all cases we use a softplus on the final layer of the variance/precision networks to ensure the output is positive.

For the $\beta-$NLL implementation we take $\beta=0.5$ as suggested in \cite{seitzer_pitfalls_2022}. 
The ensemble method we use fits 6 individual heteroskedastic neural networks and combines their outputs into a mixture distribution that is approximated with a normal distribution. We do not add in adversarial noise as the authors state it does not make a significant difference. We fit six $\beta-$NLL models and six MLE-ensembles.

\subsection{Diagonal selection criteria}\label{app:diag-crit}
After conducting our diagonal search we found the model that minimized $\mu$ MSE and the model that minimized $\Lambda^{-\frac{1}{2}}$ MSE on the \emph{training} data. 
In some cases these models coincided. 
We then used the model that was on the midpoint (on a logit scale) of the $\rho=1-\gamma$ line between these two models to compare.
The results are reported in \cref{tab:baselines}. 
In all cases our method is competitive with or exceeds the performance of these two baselines--particularly on real-world data. 
Note that our goal is to show that we are able to find models that model the mean and standard deviation of the data well, that is, lie in our proposed region \emph{S} of the phase diagram.   
We do not claim that this method will provide the globally optimal model. 

\subsection{Training Details}
Training for our method follows Appendices~\ref{app:training}~and~\ref{app:nn-training}. 
For the baselines, we use the same optimizer, gradient clipping, and cyclic learning-rate schedule.
On synthetic datasets we train for 600{,}000 epochs, and on the UCI datasets we train for 500{,}000 epochs (or 50{,}000 epochs with a batch size of 1000 for \emph{Power} due to its size).

\subsection{ClimSim Dataset}\label{app:climsimdesc}
The ClimSim dataset \citep{yu_climsim_2023} is a largescale climate dataset. 
Its input dimension is 124 and output dimension is 128. 
We use all 124 inputs to model a single output,
\emph{Visible direct solar flux, SOLS [$W/m^2$]}.
We train on 10,091,520 of the approximately 100 million points for training and we use a randomly selected 7,209 points to evaluate our models.

\begin{table*}%
\centering
\caption{Comparison of our deep heteroskedastic regression model (with diagonal 
regularization search; see Appendix~\ref{app:diag-crit}) against two baselines 
\citep{seitzer_pitfalls_2022, lakshminarayanan_simple_2017}.  
We report the mean $\pm$ standard deviation of ECE, $\mu$-MSE, and 
$\Lambda^{-1/2}$-MSE on test data, with the best mean in bold.  
MLE ensembles often fail to converge---typically through divergence of thepredicted standard deviation---producing $\inf$ or \texttt{nan} values and 
illustrating the instability of MLE-based heteroskedastic training.
}
\footnotesize
\begin{tabular}{ l *{3}{c} }
\toprule
{Dataset}&
{Mean-Variance} &
{$\beta$-NLL} &
{MLE Ensemble} \\
{\phantom{..........}Metric} & { Ours } &
{ \cite{seitzer_pitfalls_2022}} &
{\cite{lakshminarayanan_simple_2017}} \\
\midrule
Cubic & & &\\
\phantom{..........} ECE & \textbf{0.2380} ± 0.03 & 0.2385 ± 0.02 & 0.2411 ± 0.02 \\
\phantom{..........} $\mu$ MSE  & 0.2339 ± 0.01 & \textbf{0.1500} ± 0.01 & 1.1809 ± 1.88 \\
\phantom{..........} $\Lambda^{-\frac{1}{2}}$MSE  & 0.2397 ± 0.02 & \textbf{0.1397} ± 0.01 & inf ± nan     \\
Curve & & & \\
\phantom{..........} ECE  & 0.1804 ± 0.02 & \textbf{0.1754} ± 0.02 & 0.2432 ± 0.00 \\
\phantom{..........}  $\mu$ MSE & \textbf{0.4318} ± 0.12 & 0.4877 ± 0.16 & 1.0067 ± 0.19 \\
\phantom{..........} $\Lambda^{-\frac{1}{2}}$MSE  & 0.4655 ± 0.09 & \textbf{0.4187} ± 0.20 & inf ± nan     \\
Sine & & & \\
\phantom{..........} ECE   & 0.2499 ± 0.00 & \textbf{0.2082} ± 0.03 & 0.2313 ± 0.05 \\
\phantom{..........} $\mu$ MSE   & \textbf{0.7968} ± 0.00 & 4.4107 ± 6.90 & 0.9716 ± 0.06 \\
\phantom{..........} $\Lambda^{-\frac{1}{2}}$MSE   & \textbf{0.7968} ± 0.00 & 4.3524 ± 6.89 & inf ± nan     \\
\midrule
Concrete & & & \\
\phantom{..........} ECE  & 0.2471 ± 0.01 & 0.2552 ± 0.00           & \textbf{0.0655} ± 0.01           \\
\phantom{..........} $\mu$ MSE  & \textbf{0.1055} ± 0.02 & 0.5461 ± 0.30         & 2.2454 ± 1.74           \\
\phantom{..........} $\Lambda^{-\frac{1}{2}}$MSE  & \textbf{0.3028} ± 0.51 & 1.0867 ± 0.20 & $1.3 \times 10^5$ ± $1.2 \times 10^5$ \\
Housing & & & \\
\phantom{..........} ECE   & 0.0653 ± 0.00 & \textbf{0.2631} ± 0.01           & 0.1332 ± 0.02           \\
\phantom{..........} $\mu$ MSE   & \textbf{1.2236} ± 0.00 & 0.3175 ± 0.06      & 155.4494 ± 128.27       \\
\phantom{..........} $\Lambda^{-\frac{1}{2}}$MSE   & \textbf{0.7610} ± 0.00 & 0.8820 ± 0.03    & 218.8269 ± 195.38       \\
Power & & & \\
\phantom{..........} ECE    & 0.2233 ± 0.01 & 0.2370 ± 0.00           & \textbf{0.0285} ± 0.01           \\
\phantom{..........} $\mu$ MSE    & 0.0350 ± 0.01 & 0.1013 ± 0.01           & \textbf{0.0177} ± 0.00           \\
\phantom{..........} $\Lambda^{-\frac{1}{2}}$MSE    & 0.0343 ± 0.01 & 0.3081 ± 0.37           & \textbf{0.0091} ± 0.00           \\
Yacht & & & \\
\phantom{..........} ECE    & 0.3038 ± 0.04 & 0.2882 ± 0.02           & \textbf{0.0463} ± 0.02           \\
\phantom{..........} $\mu$ MSE    & \textbf{0.0077} ± 0.01 & 0.0137 ± 0.01        & 6.2670 ± 13.96          \\
\phantom{..........} $\Lambda^{-\frac{1}{2}}$MSE    & \textbf{0.0076} ± 0.01 & 1.3275 ± 0.02        & 8.0599 ± 19.18          \\
\midrule
Solar Flux & & & \\
\phantom{..........} ECE  & \textbf{0.1503} ± 0.00 & 0.3007 ± 0.00 & 0.1924 ± 0.04                    \\
\phantom{..........} $\mu$ MSE  & \textbf{0.2887} ± 0.00 & 0.3771 ± 0.00 & 1.0067 ± 0.19                    \\
\phantom{..........} $\Lambda^{-\frac{1}{2}}$MSE  & \textbf{0.1175} ± 0.00 & 0.3217 ± 0.00 & \hspace{0.85em} $4.6\times10^{9}$ ± $9.9\times 10^{9}$  \\
\bottomrule
\end{tabular}
\label{tab:baselines}
\end{table*}

\bibliographystyle{plainnat}  %
\bibliography{sample}

\end{document}